\documentclass[12pt]{colt2019}

\usepackage{xcolor}

\usepackage{todonotes}
\usepackage{amssymb}
\usepackage{algorithm}
\usepackage{algorithmic}
\usepackage{subcaption}
\usepackage{url}
\usepackage{makecell}

\newcommand{\Qbound}{\mathcal{Q}}
\allowdisplaybreaks

\newcount\Comments  
\definecolor{darkgreen}{rgb}{0,0.5,0}
\definecolor{darkred}{rgb}{0.7,0,0}
\definecolor{teal}{rgb}{0.3,0.8,0.8}
\definecolor{orange}{rgb}{1.0,0.5,0.0}
\definecolor{purple}{rgb}{0.8,0.0,0.8}
\newcommand{\kibitz}[2]{\ifnum\Comments=1{\textcolor{#1}{\textsf{\footnotesize #2}}}\fi}

\title[Contrasting Exploration in Parameter and Action Space]{Contrasting Exploration in Parameter and Action Space: \\
  A Zeroth-Order Optimization Perspective}

\coltauthor{\Name{Anirudh Vemula} \Email{vemula@cmu.edu} \\ \addr
  Robotics Institute, Carnegie Mellon University \AND \Name{Wen Sun}
  \Email{wensun@cs.cmu.edu} \\ \addr Robotics Institute, Carnegie Mellon
  University \AND \Name{J. Andrew Bagnell} \Email{dbagnell@ri.cmu.edu}
\\ \addr Aurora Innovation}

\begin{document}

\maketitle

\begin{abstract}
    Black-box optimizers that explore in parameter space have often been shown to outperform more sophisticated action space
    exploration methods developed specifically for the reinforcement learning problem. We examine these black-box methods closely to identify situations in which they are worse than action space exploration methods and those in which they are superior. Through simple theoretical analyses, we prove that complexity of exploration in parameter space depends on the dimensionality of parameter space, while complexity of exploration in action space depends on both the dimensionality of action space and horizon length. This is also demonstrated empirically by comparing simple exploration methods on several model problems, including Contextual Bandit, Linear Regression and Reinforcement Learning in continuous control.
\end{abstract}

\section{Introduction}




Model-free policy search is a general approach to learn parameterized
policies from sampled trajectories in the environment without
learning a model of the underlying dynamics. These methods
update the parameters such that trajectories with higher returns (or total reward) are
more likely to be obtained when following the updated policy
\citep{kober2013reinforcement}. The simplicity of these approaches have
made them popular in Reinforcement Learning (RL). 

Policy gradient methods, such as REINFORCE \citep{williams1992simple}
and its extensions \citep{kakade2002natural,bagnell2004policy, silver2014deterministic,schulman2015trust},
compute an estimate of a direction of improvement from sampled trajectories collected by executing a stochastic policy.
In other words, these methods rely on
randomized exploration in action space. These methods then leverage the Jacobian of the policy to update its parameters to increase the probability of good action sequences accordingly. 
Such a gradient estimation algorithm can be considered a combination of a \textit{zeroth-order} approach and a \textit{first-order} approach: (1) it never exploits the slope of the reward function or dynamics, with respect to actions, but rather relies only on random exploration in action space to discover potentially good sequences of actions; (2) however, it \emph{exploits} the first order information of the parameterized policy for updating the policy's parameters. Note that the chance of finding a sequence of actions resulting in high total reward decreases (as much as exponentially \citep{kakade2002approximately}) as the horizon length increases and thus policy gradient methods often exhibit high variance and a resulting large sample complexity \citep{peters2008reinforcement, zhao2011analysis}.

Black-box policy search methods, on the other hand, seek to directly
optimize the total reward in the space of parameters by employing
, e.g., finite-difference-like methods to compute estimates of the gradient with respect to policy parameters \citep{bagnell2001autonomous,mannor2003cross,heidrich2008evolution,tesch2011using,sehnke2010parameter,salimans2017evolution,mania2018simple}. 
Intuitively, these methods rely on exploration in parameter space: by searching in the parameter space, these methods may discover an improvement direction. Note that these methods are fully zeroth-order, \textit{i.e.}, they exploit no first-order information of the parameterized policy, the reward, or the dynamics. 
Although policy gradient methods leverage \textit{more} information,
notably the Jacobian of the action with respect to policy, black-box policy search methods have
at times demonstrated better empirical performance (see the discussion in 
\citep{kober2013reinforcement, mania2018simple}). 
These perhaps surprising results motivate us to analyze:
\textit{In
  what situations should we expect parameter space policy search methods
  to outperform action space methods?}


To do so, we leverage prior work in zeroth-order
optimization methods. In the convex setting, 
\citep{flaxman2005online, agarwal2010optimal, nesterov2017random}
showed that one can construct 
gradient 
estimates using zeroth order oracles and derived upper bounds on the
number of samples needed. 
But for most RL tasks,
the return as a function of parameters, or action sequence, is highly
non-convex \citep{sutton1998introduction}. Hence we focus on the non-convex setting and analyze convergence to stationary points. \cite{ghadimi2013stochastic, nesterov2017random} studied zeroth order
non-convex optimization by providing upper bounds on
the number of samples needed to close in on a stationary
point. Computing lower bounds in zeroth order non-convex optimization is still
an open problem 
\citep{carmon2017lower, carmon2017lower2}. 

In our work, we extend the analysis proposed in \citep{ghadimi2013stochastic} to the policy search setting 
and analyze the sample complexity of parameter and action space exploration methods in policy search. We begin with a degenerate, one-step control problem of online linear regression with partial feedback, \citep{flaxman2005online}, where the objective is to learn the parameters of the linear regressor without access to the true scalar regression targets. We show that for parameter space exploration methods, to achieve $\epsilon$-optimality, requires $O(b^2/\epsilon^4)$ samples, 
where $b$ is the input feature dimensionality. By contrast, an action space exploration method requires $O(1/\epsilon^4)$ many samples
with a sample complexity \textit{independent} of input feature dimensionality $b$. 
This is tested empirically on two simple tasks:  Bandit Multi-class learning on MNIST with policies parameterized by convolutional neural networks which can be seen as a Contextual Bandit problem with rich observations, and Online Linear Regression with partial information. The results demonstrate action space exploration methods outperform  parameter space methods when the parameter dimensionality is substantially larger than action dimensionality.

We present similar analysis for the multi-step control problem of model-free policy search in reinforcement learning, \citep{kober2013reinforcement}, by considering the objective of reaching $\epsilon$-close to a stationary point in the sense that $\|\nabla J(\theta)\|_2^2 \leq \epsilon$ for the non-convex objective $J(\theta)$. Our results show that, under certain assumptions, parameter space exploration methods 
need $\mathcal{O}(\frac{d^2}{\epsilon^3})$ samples to reach $\epsilon$ close to a stationary point, where $d$ is the policy parameter dimensionality. On the other hand, action space exploration methods need $\mathcal{O}(\frac{p^2H^4}{\epsilon^4})$ samples to achieve the same objective, where $p$ is the action dimensionality and $H$ is the horizon length of the task. This shows that action space exploration methods have a dependence on the horizon length $H$ 
while parameter space exploration methods depend only on parameter space dimensionality $d$.
Ongoing work by \cite{tu2018gap} demonstrated through asymptotic lower bounds that the dependence of sample complexity of action space exploration methods on horizon $H$ is unavoidable in the LQR setting.
This is tested empirically on popular RL benchmarks from OpenAI gym \citep{openaigym}, and the results show that as horizon length increases, parameter space methods outperform action space exploration methods. This matches the intuition and results presented in recent works like \citep{bagnell2001autonomous,szita2006learning,tesch2011using,salimans2017evolution, mania2018simple}
that show parameter space black-box policy search methods outperforming state-of-the-art action space methods for tasks with long horizon lengths.

In summary, our analysis and experimental results suggests that the
complexity of exploration in action space depends on both the
dimensionality of action space and \emph{horizon}, while the
complexity of exploration in parameter space solely depends on
dimensionality of parameter space, providing a natural way to trade-off
between these approaches.







\section{Problem Setup}
\label{sec:problem_define}

\subsection{Multi-step Control: Reinforcement Learning}
\label{sec:rl-problem}
We consider the problem setting of model-free policy search with the goal of
minimizing sum of costs (or maximizing sum of rewards) over a fixed,
finite horizon $H$. In reinforcement learning (RL), this is typically
formulated using Markov Decision Processes (MDP)
\citep{sutton1998introduction}. Denote the state space of the MDP as
$\mathcal{S} \subset 
\mathbb{R}^b$, action space as $\mathcal{A} \subset  \mathbb{R}^p$
, transition probabilities as
$\mathbb{P}_{sa} = \mathbb{P}(\cdot|s, a)$ (which is the distribution of next state after
executing action $a \in \mathcal{A}$ in state $s \in \mathcal{S}$), an
initial state distribution $\mu$, and a cost function $c(s,a) :
\mathcal{S} \times \mathcal{A} \rightarrow \mathbb{R}$. Note that the
cost can be interpreted as negative of the reward. In addition to
this, we assume a restricted class of deterministic, stationary
policies $\Pi$ parameterized by $\theta \in \mathbb{R}^d$ where each
$\pi(\theta, \cdot) \in \Pi$ is differentiable at all $\theta$ and is a mapping from $\mathcal{S}$ to
$\mathcal{A}$, i.e. $\pi(\theta, \cdot): \mathcal{S} \rightarrow
\mathcal{A}$. The distribution of states at timestep $t$ induced by
running the policy $\pi(\theta, \cdot)$ until and including $t$, is
defined $\forall s_t: d_{\pi_\theta}^t(s_t) = \sum_{\{s_i\}_{i \leq
    t-1}} \mu(s_0) \prod_{i=0}^{t-1} \mathbb{P}(s_{i+1}|s_i, a_i =
\pi(\theta, s_i))$, where by definition $d_{\pi_\theta}^0(s) = \mu(s)$
for any $\pi$. We define the value function $V^t_{\pi_\theta}(s)$ for $t \leq H-1$ as 
\begin{equation*}
  V^t_{\pi_\theta}(s) = \mathbb{E}[\sum_{i=t}^H c(s_i, \pi(\theta,
s_i))|s_t = s]
\end{equation*}
and state-action value function $Q^t_{\pi_\theta}(s, a)$ as
\begin{equation*}
  Q^t_{\pi_\theta}(s, a) = c(s, a) + \mathbb{E}_{s' \sim
                           \mathbb{P}_{sa}}[V^{t+1}_{\pi_\theta}(s')]
\end{equation*}
Throughout this work, we assume the total cost is upper bounded by a constant, i.e., $\sup_{c_1,\dots, c_T}\sum_{t} c_t \leq \Qbound\in\mathbb{R}^+$, to prevent confounding due to just a change in the scale of total costs. We have then that $Q^t_{\pi_\theta}$ is upper bounded by a constant $\Qbound$ for all $t$ and $\theta$.

We seek to minimize
the performance objective given by $J(\theta) = \mathbb{E}_{s \sim
  \mu}[V^0_{\pi_\theta}(s)]$. Given this
objective, the optimization problem  can be formulated
as:
\begin{equation}
  \label{eq:1}
  \min_\theta J(\theta)
\end{equation}
The goal is to find parameters $\theta^*$ that minimize the
expected sum of costs $J(\theta)$, given no access to the underlying dynamics
of the environment other than samples from the distribution
$\mathbb{P}_{sa}$ by executing the policy $\pi(\theta,
\cdot)$. However, the objective $J(\theta)$ can be highly non-convex
and finding a global minima could be intractable. Thus, in this work, we hope to find a stationary point
$\theta^*$ of the objective $J(\theta)$, \textit{i.e.} a point where
$\nabla_\theta J(\theta) \approx 0$.


\subsection{One-Step Control: Online Linear Regression with Partial Information}
\label{sec:define_OLR}
The online linear regression problem is defined as follows: We denote
$\mathcal{S}\subset\mathbb{R}^{b}$ as the feature space, and $\Theta
\subset\mathbb{R}^{d} = \mathbb{R}^b$ as the linear policy parameter space where each
$\theta\in\Theta$ represents a policy $\pi(\theta, s) =
\theta^{\top}s$. Online linear regression operates in an adversarial
online learning fashion: every round $i$, nature presents a
feature vector $s_i\in\mathcal{S}$, the learner makes a decision by
choosing a policy $\theta_{i}\in\Theta$ and predicts the scalar action
$\hat{a}_i = \theta_i^{\top}s_i$; nature then reveals the loss
$(\hat{a}_i - a_i)^2 \in\mathbb{R}^+$, which is just a scalar, to the
learner, where $a_i$ is ground truth selected by nature and is
never revealed to the learner.  We do not place any statistical
assumption on the nature's process of generating feature vector $s_i$
and ground truth $a_i$, which could be completely adversarial. Other
than the adversarial aspect of the problem, note that the above setup
is a special setting of RL with horizon $H=1$, linear policy
$\theta^{\top}s_i$, one-dimension action space, and a cost function
$c_i(\theta) = (\theta^{\top}s_i - a_i)^2$. In this setting, we
consider the \textit{regret} with respect to the optimal solution in hindsight,
\begin{align}
    \mathrm{Regret} = \sum_{i=1}^{T} c_i(\theta_i) - \min_{\theta^\star\in\Theta} \sum_{i=1}^T c_i(\theta^\star)
\end{align}

\section{Online Linear Regression with Partial Information}
\label{sec:olr}

\subsection{Exploration in Parameter Space}
We can apply a zeroth-order
online gradient descent algorithm for the sequence of loss functions
$\{c_i\}_{i=1}^T$, which is summarized in
Algorithm~\ref{alg:random_search_OLR}. The main idea is to add random
noise $u$, sampled from a unit sphere in $b$-dim space $\mathbb{S}_{b}$, to the parameter $\theta$, and querying loss at $\theta+\delta u$ for some $\delta > 0$. Using the received loss $c_i(\theta+\delta u)$, one can form an estimation of $\nabla_{\theta}c_i(\theta)$ as $\frac{c_i b}{\delta}u$ \citep{flaxman2005online}. 


\begin{algorithm}[ht]
\caption{Random Search in Parameter Space (BGD \cite{flaxman2005online})}
 \label{alg:random_search_OLR}
\begin{algorithmic}[1]
  \STATE {\bfseries Input:} $\alpha\in\mathbb{R}^+$, $\delta\in\mathbb{R}^+$.
  \STATE Learner initializes $\theta_1\in\Theta$.
  \FOR {$i = 1$ to $T$}
    \STATE Learner samples $u\sim \mathbb{S}_{b}$. 
    \STATE Learner chooses predictor $\theta_i' = \theta_i + \delta u$. 
    \STATE Learner only receives loss signal $c_i(\theta_i')$. 
    \STATE Learner update: $\theta_{i+1}' = \theta_i - \alpha \frac{c_i b}{\delta}u$.
    \STATE Projection $\theta_{i+1} = \arg\min_{\theta\in{\Theta}}\|\theta_{i+1}'-\theta\|_2^2$.
  \ENDFOR
\end{algorithmic}
\end{algorithm}

\subsection{Exploration in Action Space}
The \emph{key difference} between exploration in action space and exploration in parameter space is that we are going to leverage our knowledge of the policy $\pi(\theta, s) = \theta^{\top} s$. Since we design the policy class, we can compute its \emph{Jacobian} with respect to its parameters $\theta$ without interaction with the environment. The Jacobian of the policy gives us a locally linear relationship between a small change in parameter and the resulting change in policy's action space. The main idea then in this approach is to explore with randomization in action space, and then leverage the Jacobian of the policy to update the parameters $\theta$ accordingly so that the policy's output moves towards better actions. Intuitively, we expect that random exploration in action space will result in smaller regret, as in our setting the action space is just $1$-dimensional, while the parameter space is $b$-dimensional.  
The approach is summarized in Algorithm~\ref{alg:random_search_action_olr}. Denote $\ell_i = (\hat{a}_i - a_i)^2$ and $\hat{a}_i = \pi(\theta_i, s_i) = \theta_i^{\top}s_i$. The main idea is that we can compute $\nabla_{\theta} c_i(\theta_i)$ via a chain rule as $\nabla_{\theta} c_i (\theta_i) = \frac{\partial{\ell_i}}{\partial{\hat{a}_i}}\nabla_{\theta}\pi(\theta_i, s_i)$. Note that $\nabla_{\theta}\pi(s_i, \theta_i) = \nabla_\theta \theta_i^{\top}s_i = s_i$ is the Jacobian of the policy to which we have full access. We then use zeroth order approximation method to approximate $\partial{\ell_i}/\partial{\hat{a}_i}$ at $\hat{a}_i = \pi(\theta_i, s_i)$.

\begin{algorithm}[ht]
\caption{Random Search in Action Space}
 \label{alg:random_search_action_olr}
\begin{algorithmic}[1]
  \STATE {\bfseries Input:} $\alpha\in\mathbb{R}^+$, $\delta\in\mathbb{R}^+$.
  \STATE Learner initializes $\theta_1\in\Theta$.
  \FOR {$i = 1$ to $T$}
    \STATE Learner receives feature $s_i$.
    \STATE Learner samples $e$ uniformly from $\{-1,1\}$.
    \STATE Learner makes a prediction $\hat{a}_i = \theta_i^{\top} s_i + \delta e$
    \STATE Learner only receives loss signal $c_i = (\hat{a}_i - a_i)^2$. 
    \STATE Learner update: $\theta_{i+1}' = \theta_i - \alpha \frac{c_i e}{\delta}s_i$.
    \STATE Projection $\theta_{i+1} = \arg\min_{\theta\in{\Theta}}\|\theta_{i+1}'-\theta\|_2^2$.
  \ENDFOR
\end{algorithmic}
\end{algorithm}

\subsection{Analysis}
We analyze the regret of the exploration in parameter space algorithm (Alg.~\ref{alg:random_search_OLR}) and the exploration in action space algorithm (Alg.~\ref{alg:random_search_action_olr}) in this section. For analysis, we assume that $\Theta$ is bounded, i.e., $\sup_{\theta\in\Theta}\|\theta\|_2 \leq C_{\theta}\in\mathbb{R}^+$, $\mathcal{S}$ is bounded, i.e., $\sup_{s\in\mathcal{S}}\|s\|_2 \leq C_s\in\mathbb{R}^+$, and the ground truth $a_i$ is bounded, i.e., $|a_i|\leq C_{a}$ for any $i$. Under the above assumptions, we can make sure that the loss is bounded as well, $(\theta^{\top}s - a)^2 \leq C\in\mathbb{R}^{+}$. The loss function is also Lipschitz continuous with Lipschitz constant $L \leq (C_{\theta}C_{s} + C_{a})C_{s}$. We call these constants $C_{s}, C_{\theta}$, and $C_{a}$ as problem dependent constants, which are independent of feature dimension $b$ and number of rounds $T$.
In regret bounds, we absorb problem dependent constants into $\mathcal{O}$ notations, but the bounds will be explicit in $b$ and $T$. The theorem below presents the average regret analysis for these methods,

\begin{theorem}
After $T$ rounds, 
with $\alpha = \frac{{C_{\theta}}\delta}{b(C^2+C_{s}^2)\sqrt{T}}$ and $\delta = T^{-0.25}\sqrt{\frac{C_{\theta}b(C^2+C_{s}^2)}{2L}}$, Alg.~\ref{alg:random_search_OLR} incurs average regret:
\begin{align}
\label{eq:random_para}
    \frac{1}{T}(\mathbb{E}[\sum_{i=1}^T c_i(\theta_i)] - \min_{\theta^\star
  \in \Theta}\sum_{i=1}^T c_i(\theta^\star)) \leq 
    \mathcal{O}(\sqrt{b}T^{-\frac{1}{4}}),
\end{align}
and with $\alpha = \frac{C_{\theta}\delta}{(C^2+1)C_{s}\sqrt{T}}$ and $\delta = T^{-0.25}\sqrt{\frac{C_{\theta}(C^2+1)C_{s}}{2C}}$, Alg.~\ref{alg:random_search_action_olr} incurs average regret:
\begin{align}
\label{eq:random_action}
    \frac{1}{T}(\mathbb{E}[\sum_{i=1}^T c_i(\theta_i)] -
  \min_{\theta^\star \in \Theta}\sum_{i=1}^T c_i(\theta^\star)) \leq 
    \mathcal{O}(T^{-\frac{1}{4}}),
\end{align} for any $\theta\in\Theta$.
\label{thm:online_linear_regression}
\end{theorem}
The above regret analysis essentially shows that exploration in action space delivers a regret bound that is independent of parameter space dimension $b$, while the regret of the exploration in parameter space algorithm will have explicit polynomial dependency on feature dimension $b$.  Converting the regret bounds to sample complexity bounds, we have that for any $\epsilon\in (0,1)$, to achieve $\epsilon$-average regret, Alg.~\ref{alg:random_search_OLR} needs $\mathcal{O}(\frac{b^2}{\epsilon^4})$ many rounds, while Alg.~\ref{alg:random_search_action_olr} requires $\mathcal{O}(1/\epsilon^4)$ many rounds.  

Note that in general if we have a multivariate regression problem, i.e., $a\in\mathbb{R}^{p}$, regret of Algorithm~\ref{alg:random_search_action_olr} will depend on $\sqrt{p}$ as well. But from our extreme case with $p=1$, we clearly demonstrate the sharp advantage of exploration in action space: \emph{when the action space's dimension is smaller than the dimension of parameter space}, we should prefer the strategy of exploration in action space.

\section{Reinforcement Learning}
\label{sec:RL}

In this section, we study exploration in parameter space versus exploration in action space for multi-step control problem of model-free policy search in RL. As explained in Section~\ref{sec:problem_define}, we are interested in rates of convergence to a stationary point of $J(\theta)$.

\subsection{Exploration in Parameter Space}
\label{sec:parameter_space}
The objective defined in Section \ref{sec:rl-problem} can be optimized
directly over the space of parameters $\mathbb{R}^d$. Since we do not use first-order (or gradient) information about the
objective, this is equivalent to derivative-free (or zeroth-order)
optimization with noisy function evaluations. More specifically, for a parameter vector
$\theta$, we can execute the corresponding policy $\pi(\theta, \cdot)$
in the environment, to obtain a noisy estimate of $J(\theta)$. This
noisy function evaluation can be used to construct a gradient estimate
and an iterative stochastic gradient descent approach can be used to
optimize the objective. An algorithm that closely follows the ones
proposed in \citep{agarwal2010optimal, mania2018simple}
 and optimizes over the space of parameters  is shown in Algorithm
\ref{alg:random_search_parameter}. Since we are working in episodic RL setting, we can use a two-point estimate to form a gradient estimation (Line 7 \& 8 in Alg.~\ref{alg:random_search_parameter}), which in general will reduce the variance of gradient estimation \citep{agarwal2010optimal}, compared to one-point estimates.
\begin{algorithm}[ht]
\caption{Policy Search in Parameter Space}
 \label{alg:random_search_parameter}
\begin{algorithmic}[1]
  \STATE {\bfseries Input:} Learning rate $\alpha \in\mathbb{R}^+$, standard deviation of exploration noise $\delta\in\mathbb{R}$
  \STATE Initialize parameters $\theta_1\in\mathbb{R}^d$
  \FOR {$i = 1$ to $T$}
    \STATE Sample $u \sim \mathbb{S}_d$ , a $d$-dimensional unit sphere
    \STATE Construct parameters $\theta_i + \delta u$, $\theta_i - \delta u$
    \STATE Execute policies $\pi({\theta_i + \delta u}, \cdot),
    \pi({\theta_i - \delta u}, \cdot)$ 
    \STATE Obtain noisy estimates of the objective $J^+_i = J(\theta_i
    + \delta u) + \eta^+_i$ and $J^-_i = J(\theta_i - \delta u) +
    \eta^-_i$ where $\eta^+_i, \eta^-_i$ are zero mean random i.i.d noise
    \STATE Compute gradient estimate $g_i = \frac{d(J^+_i - J^-_i)}{2\delta} u $
    \STATE Update $\theta_{i+1} = \theta_i - \alpha g_i$
  \ENDFOR
\end{algorithmic}
\end{algorithm}
We will analyze the finite rate of convergence of Algorithm \ref{alg:random_search_parameter} to a stationary point of the non-convex objective $J(\theta)$. First, we will lay out the assumptions and then present the convergence analysis. 
\paragraph{Assumptions and Analysis}
\label{sec:assumptions_parameter}

To analyze convergence to stationary point of a nonconvex objective, we make several assumptions about the objective. Firstly, we assume that $J(\theta)$ is differentiable with respect to $\theta$ over the entire domain. We also assume that $J(\theta)$ is $G$-lipschitz and $L$-smooth, i.e. for all $\theta_1, \theta_2 \in \mathbb{R}^d$, we have $|J(\theta_1) - J(\theta_2)| \leq G\|\theta_1 - \theta_2\|$ and $\|\nabla_\theta J(\theta_1) - \nabla_\theta J(\theta_2)\|\leq L\|\theta_1 - \theta_2\|$. 
Note that these assumptions are similar to the assumptions made in other zeroth-order analysis works, \citep{flaxman2005online, agarwal2010optimal, duchi2015optimal, shamir2013complexity, ghadimi2013stochastic, nesterov2017random}.


Our analysis is along the lines of works like \citep{ghadimi2013stochastic,
  nesterov2017random} that also analyze the convergence to stationary points in
zeroth order non-convex optimization. The general strategy is to first construct a smoothed version of the objective $J(\theta)$, denoted as $\hat{J}(\theta) = \mathbb{E}_{v\sim \mathbb{B}_{d}}[J(\theta+\delta v)]$, where $\mathbb{B}_d$ is the $d$-dimensional unit ball. We can then show that Algorithm~\ref{alg:random_search_parameter} is essentially running SGD on the objective function $\hat{J}(\theta)$, which allows us to apply standard SGD analysis on $\hat{J}(\theta)$. Lastly we link the stationary point of the smoothed objective $\hat{J}(\theta)$ to that of the objective $J(\theta)$ using the assumptions on ${J}(\theta)$.

\begin{theorem}
  \label{theorem:parameter-convergence}
  Consider running Algorithm \ref{alg:random_search_parameter} for $T$
  steps where the true objective $J(\theta)$ satisfies the assumptions stated above.
  Then we have,
  \begin{equation}
    \label{eq:parameter-convergence}
    \frac{1}{T}\sum_{i=1}^T \mathbb{E}\|\nabla_\theta
    J(\theta_i)\|_2^2 \leq \mathcal{O}(\Qbound^{\frac{1}{2}}dT^{\frac{-1}{2}} + \Qbound^{\frac{1}{3}}d^{\frac{2}{3}}T^{\frac{-1}{3}}\sigma)
  \end{equation}
  where $J(\theta)\leq \Qbound$ for all $\theta \in \Theta$ and $\sigma^2$ is the variance of the random noise $\eta$ in Algorithm \ref{alg:random_search_parameter}.
\end{theorem}
The above theorem gives us a convergence rate to a stationary point of
policy search in parameter space. The role of variance of i.i.d noise
in the noisy evaluations of the true objective is very
important. Consider the case where there is little stochasticity in
the environment dynamics, i.e. $\sigma \rightarrow 0$, then the first term in
Equation \ref{eq:parameter-convergence} becomes dominant and we only
need at most $\mathcal{O}(\frac{d^2\Qbound}{\epsilon^2})$ samples to reach a point $\theta$
where $\mathbb{E}\|\nabla_\theta J(\theta)\|_2^2 \leq
\epsilon$. However, if there is a lot of stochasticity in the
environment dynamics then the second term is dominant and we need at most
$\mathcal{O}(\frac{d^2\Qbound\sigma^3}{\epsilon^3})$ samples. It is
interesting to observe the direct impact that the stochasticity of
environment dynamics has on convergence rate of policy search, which
is also experimentally demonstrated in
Sec.~\ref{sec:effect-stoch-param}. Note that the convergence rate has
no dependency on horizon length $H$ because of the regularity
assumption we used on total reward: $J$ is always bounded by a
constant $\Qbound$ that is independent of $H$. However, as we will see
later, even under the regularity assumption convergence rate of action
space exploration methods have an explicit dependence on $H$
which will prove to be the primary
reason why black-box parameter space policy search methods in \citep{mania2018simple} have been so effective when compared to action space methods.

\subsection{Exploration in Action Space}
\label{sec:action_space}

Another way to optimize the objective defined in Section \ref{sec:rl-problem}
is to optimize over the space of actions $\mathcal{A}$. From \citep{silver2014deterministic}, we know that for
$J(\theta) = \mathbb{E}_{s \sim \mu}[V^0_{\pi_\theta}(s)]$ we can express the
gradient as
\begin{align}
  \label{eq:dpg-gradient}
  \nabla_\theta J(\theta) &= \sum_{t=0}^{H-1} \mathbb{E}_{s_t \sim
    d^t_{\pi_\theta}} \left[\nabla_\theta \pi(\theta, s_t) \nabla_a
  Q^t_{\pi_\theta}(s_t, \pi(\theta, s_t))\right]
\end{align}
Observe that the first term in the above gradient $\nabla_\theta
\pi(\theta, s)$ is the Jacobian of the policy, the local linear relationship
between  a small change in policy parameters $\theta$ and a small change in its output, i.e., actions. The second term $\nabla_{a}Q(s,a)$ is actually the improvement direction at state action pair $(s,a)$, i.e., conditioned on state $s$, if we move action $a$ an infinitesimally small step along the negative gradient $-\nabla_{a}Q(s,a)$, we decrease the cost-to-go $Q(s,a)$. Eqn~\ref{eq:dpg-gradient} then leverages policy's Jacobian to transfer the improvement direction in action space to an improvement direction in parameter space.  


We can compute Jacobian $\nabla_{\theta}\pi(\theta,s)$ exactly as we have knowledge of the policy function, i.e, we can leverage the first-order information of the parameterized policy. The
second term $\nabla_a Q_{\pi_\theta}^t(s, \pi(\theta, s_t))$, however, is unknown as it depends on the dynamics and cost functions and
needs to be estimated by interacting with the environment. We could
employ a similar algorithm as Algorithm \ref{alg:random_search_parameter}, shown in Algorithm \ref{alg:random_search_action}, to
obtain an estimate of the gradient $\nabla_a Q_{\pi_\theta}^t(s,
\pi(\theta, s_t))$, i.e., a zeroth order estimation of $\nabla_{a}Q^{t}_{\pi_{\theta}}$, computed as $\frac{p \tilde{Q}_i}{\delta}u$, where $\tilde{Q}_i$ is an unbiased estimate of ${Q}^{t}_{\pi_{\theta_i}}(s_t, \pi(\theta_i,s_t)  + \delta u)$, with $u\sim \mathbb{S}_{p}$ (Line 7 \& 9 in Alg.~\ref{alg:random_search_action}).

Another important difference from Algorithm
\ref{alg:random_search_parameter} is the fact that we use a one-point
estimate for the gradient $g_i$ in Algorithm
\ref{alg:random_search_action}. We cannot employ the idea of two-point estimate  in random exploration
in action space to reduce the variance of the estimate of $\nabla_{a}Q^{t}_{\pi_{\theta}}(s_t,a)$. This is due to the fact that environment is stochastic, and we cannot guarantee that we will reach the same state $s_t$ at any two independent roll-ins with $\pi_{\theta}$ at time step $t$.
\begin{algorithm}[ht]
\caption{Policy Search in Action Space}
 \label{alg:random_search_action}
\begin{algorithmic}[1]
  \STATE {\bfseries Input:} Learning rate $\alpha \in\mathbb{R}^+$,
  standard deviation of exploration noise $\delta\in\mathbb{R}$,
  Horizon length $H$, Initial state distribution $\mu$
  \STATE Initialize parameters $\theta_1\in\mathbb{R}^d$
  \FOR {$i = 1$ to $T$}
    \STATE Sample $u \sim \mathbb{S}_p$ , a $p$-dimensional unit
    sphere
    \STATE Sample uniformly $t \in \{0, \cdots, H-1\}$
    \STATE Execute policy $\pi(\theta_i, \cdot)$ until $t-1$ steps
    \STATE Execute perturbed action $a_t = \pi(\theta_i, s_t) + \delta
    u$ at timestep $t$ and continue with policy $\pi(\theta_i, \cdot)$
    until timestep $H-1$
    to obtain an estimate $\tilde{Q}_i = Q^t_{\pi_{\theta_i}}(s_t,
    \pi(\theta_i, s_t) + \delta u) + \tilde{\eta}_i$ where
    $\tilde{\eta}_i$ is zero mean random noise
    \STATE Compute policy Jacobian
    $\Psi_i = \nabla_\theta \pi(\theta_i, s_t)$
    \STATE Compute gradient estimate $g_i = H\Psi_i\frac{p\tilde{Q}_i}{\delta}u$
    \STATE Update $\theta_{i+1} = \theta_i - \alpha g_i$
  \ENDFOR
\end{algorithmic}
\end{algorithm}
Similar to Section \ref{sec:parameter_space}, we will analyze the rate of convergence of Algorithm \ref{alg:random_search_action}
to a stationary point of the objective $J(\theta)$. The following
section will lay out the assumptions and  present the convergence
analysis.

\paragraph{Assumptions and Analysis}
\label{sec:assumptions_action}

The assumptions for policy search in action space are similar to the
assumptions in Section \ref{sec:assumptions_parameter}. We assume that
$J(\theta)$ is differentiable with respect to $\theta$ over the entire domain. We
also assume that $J(\theta)$ is $G$-lipschitz and $L$-smooth. In
addition to these assumptions, we will assume that the policy function
$\pi(\theta, s)$ is $K$-lipschitz in $\theta$ and the state-action
value function $Q^t_{\pi_\theta}(s, a)$ is $W$-lipschitz and $U$-smooth in $a$. Finally, we assume that the state-action value function $Q(s,a)$ is differentiable with respect to $a$ over the entire domain.  Note that the Lipschitz assumptions above on $J(\theta)$, $Q^{t}_{\pi_{\theta}}(s,a)$, and $\pi(\theta,s)$ are also used in the analysis of Deterministic policy gradient \citep{silver2014deterministic}. We need extra smoothness assumption to study the convergence of our algorithms. 

Note that the gradient estimate $g_i$ used in
Algorithm \ref{alg:random_search_action} is a biased estimate
of $\nabla_\theta J(\theta)$. We can show this by considering
\begin{align*}
  \mathbb{E}_i[g_i] = \mathbb{E}_t \mathbb{E}_{s_t \sim
  d_{\pi_{\theta_i}}^t}\left[H\nabla_\theta \pi(\theta_i, s_t) \mathbb{E}_{u
  \sim \mathbb{S}_p}\left[\frac{p\tilde{Q}_i}{\delta}u\right]\right]
\end{align*}
where $\mathbb{E}_i$ denotes expectation with respect to the randomness at iteration $i$.
From \citep{flaxman2005online}, we have that $\mathbb{E}[\frac{p\tilde{Q}_i}{\delta}u] = \nabla_a \mathbb{E}_{v
\sim \mathbb{B}_p}[Q_{\pi_{\theta_i}}^t(s_t, \pi(\theta_i, s_t) + \delta v)]$ so we can rewrite the above equation as
\begin{align*}
  \mathbb{E}[g_i]
  &= \sum_{t=0}^{H-1}\mathbb{E}_{s_t \sim
    d_{\pi_{\theta_i}}^t}\mathbb{E}_{v \sim
    \mathbb{B}_p}\left[\nabla_\theta \pi(\theta_i, s_t)\nabla_a
    Q_{\pi_{\theta_i}}^t(s_t, \pi(\theta_i, s_t) + \delta v)\right]
\end{align*}

Comparing the above expression with equation \ref{eq:dpg-gradient}, we can see that $g_i$ is not an unbiased
estimate of the gradient $\nabla_\theta J(\theta)$. We can also explicitly upper bound the variance of $g_i$ by $\mathbb{E}_i\|g_i\|_2^2$. Note that in the limit when $\delta\to 0$, $g_i$ becomes an unbiased estimate of $\nabla_{\theta}J(\theta)$, but the variance will approach to infinity. In our analysis, we explicitly tune $\delta$ to balance the bias and variance. 

\begin{theorem}
\label{theorem:action-convergence}
  Consider running Algorithm \ref{alg:random_search_action} for $T$
  steps where the objective $J(\theta)$ satisfies the assumptions stated above. Then, we have 
  \begin{equation}
    \label{eq:action-convergence}
    \frac{1}{T}\sum_{i=1}^T \mathbb{E}\|\nabla_\theta
    J(\theta_i)\|_2^2 \leq \mathcal{O}(T^{-\frac{1}{4}}Hp^{\frac{1}{2}}(\Qbound^3 + \sigma^2\Qbound)^{\frac{1}{4}})
  \end{equation}
  where $J(\theta)\leq \Qbound$ for all $\theta \in \Theta$ and $\sigma^2$ is the variance of the random noise $\tilde{\eta}$ in Algorithm \ref{alg:random_search_action}.
\end{theorem}

The above theorem gives us a convergence rate to a stationary point of
$J(\theta)$ for policy search in action space. This means that to
reach a point $\theta$ where $\mathbb{E}\|\nabla_\theta
J(\theta)\|_2^2 \leq \epsilon$, policy search in action space needs at
most $\mathcal{O}\left( \frac{p^2H^4}{\epsilon^4} (\Qbound^3 + \sigma^2\Qbound) \right)$
samples. Interestingly, the convergence rate has a
dependence on the horizon length $H$, unlike policy search in
parameter space.
Also, observe that the 
convergence rate has no dependence on the parameter dimensionality $d$ as we have complete
knowledge of the Jacobian of policy, and  we have a dependence
on stochasticity of the environment $\sigma$ that slows down the convergence as the stochasticity 
increases, similar to policy search in parameter space.

\section{Experiments}
\label{sec:experiments}
Given the analysis presented in the previous sections, we test the
convergence properties of parameter and action space policy search
approaches across several experiments: Contextual Bandit with rich observations, Linear Regression, RL
benchmark tasks and Linear Quadratic Regulator (LQR). We use Augmented Random Search (ARS), from 
\citep{mania2018simple}, as the policy search in parameter space
method in our experiments as it has been empirically shown to be
effective in RL tasks. For policy search in action space, we use
either REINFORCE 
\citep{williams1992simple}, or ExAct (Exploration in Action Space), the method described by
Algorithm \ref{alg:random_search_action}.
In all the plots shown, solid lines
represent the mean estimate over $10$ random seeds and shaded regions
correspond to $\pm 1$ standard error.
The code for all our experiments can be found here\footnote{\url{https://github.com/LAIRLAB/contrasting_exploration_rl}}\footnote{\url{https://github.com/LAIRLAB/ARS-experiments}}.

\subsection{One-Step Control}
\label{sec:one-step-control}

In these sets of experiments, we test the convergence rate of policy
search methods for one time-step prediction. The objective is to minimize the instantaneous cost incurred. The motivation behind such experiments is that we want to
understand the dependence of different policy search methods on
parametric dimensionality $d$ without the effect of horizon length
$H$.


\paragraph{MNIST as a Contextual Bandit} Our first set of experiments is the MNIST digit recognition task 
\citep{lecun1998gradient}. To formulate the task in an RL framework, we
consider a sequential decision making problem where at each time-step
the agent is given the features of the image and needs to predict one
of ten actions (corresponding to digits). A reward of $+1$
is given for predicting the correct digit, and a reward of $-1$ for
an incorrect prediction. With this reduction, the problem is essentially a Contextual Bandit Problem \citep{agarwal2014taming}. We use a standard LeNet-style convolutional architecture,
\citep{lecun1998gradient}, with $d=21840$ trainable parameters.
\begin{figure}[t]
  \centering
  \includegraphics[width=0.5\linewidth]{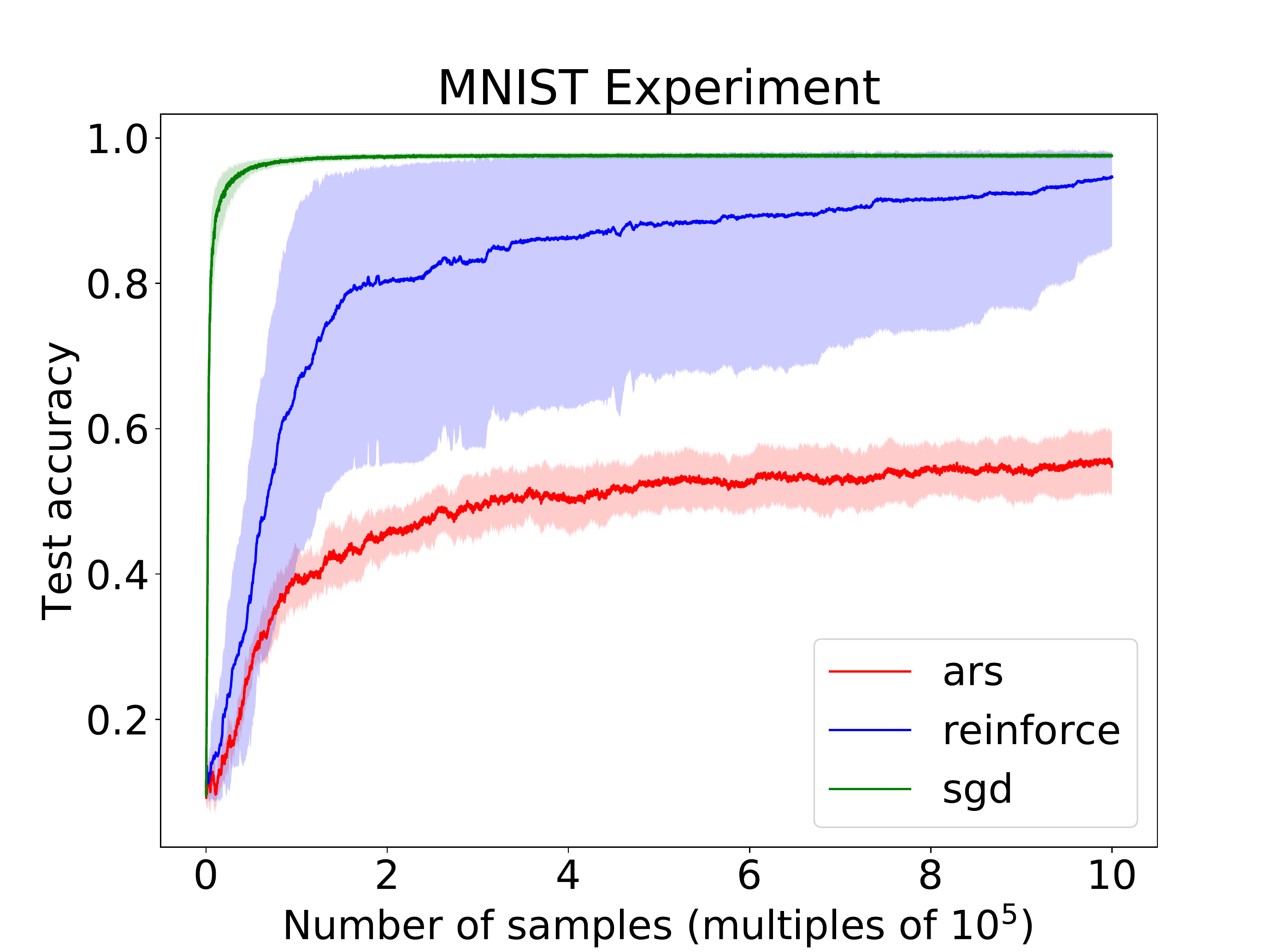}
  \caption{Mean test accuracy with standard error for different
    approaches against number of samples}
  \label{fig:mnist}
\end{figure}
Figure \ref{fig:mnist} shows the learning curves for SGD under standard full-information
supervised learning setting with cross entropy loss, REINFORCE and ARS. 
We can
observe that in this setting where the parameter space dimensionality
$d$ significantly exceeds the action space dimensionality $p = 1$, policy search in
action space outperforms parameter space methods.


\paragraph{Linear Regression with Partial Information} These set of experiments are designed to understand how the sample
complexity of different policy search methods vary as the parametric
complexity is varied. More specifically, from our analysis in Section \ref{sec:olr}, we know
that sample complexity of parameter space methods have a dependence on
$d$, the parametric complexity, whereas action space methods have no
dependence on $d$. We test this hypothesis in this experiment using artificial data
with varying input dimensionality and output scalar values.
\begin{figure*}[t]
  \centering
  \begin{subfigure}{\includegraphics[width=0.32\linewidth]{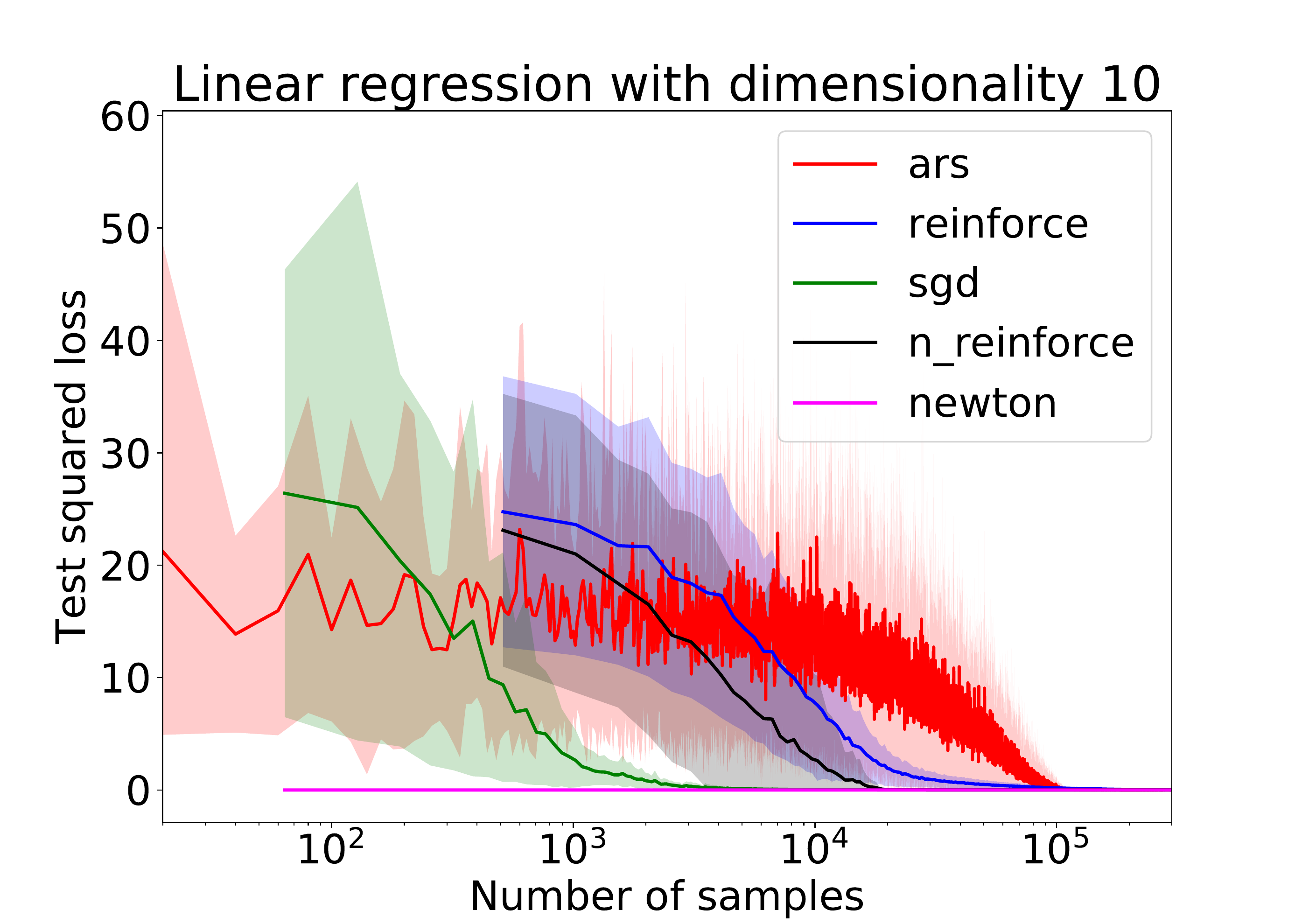}\label{fig:lin10}}\end{subfigure}
  \begin{subfigure}{\includegraphics[width=0.32\linewidth]{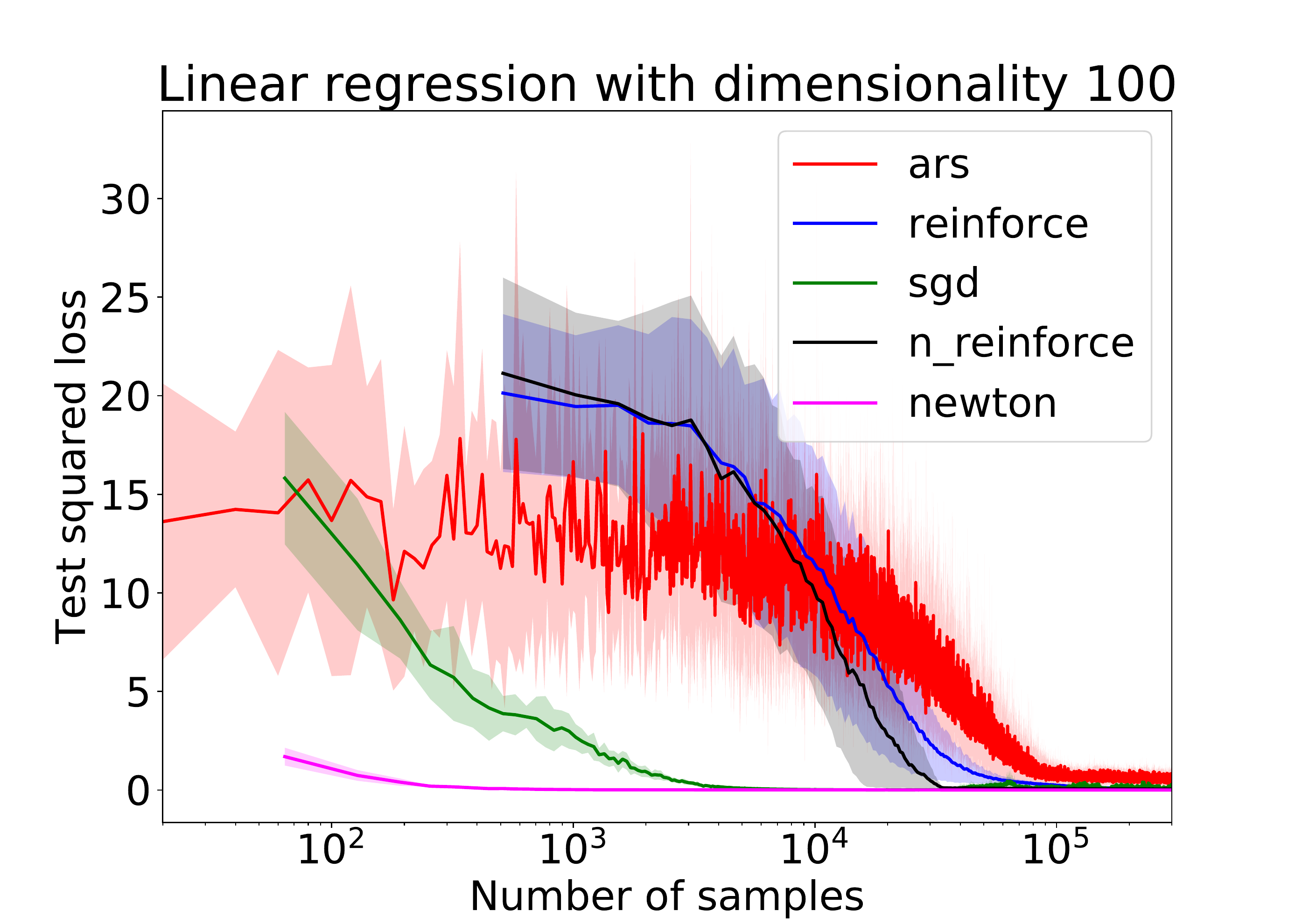}\label{fig:lin100}}\end{subfigure}
  \begin{subfigure}{\includegraphics[width=0.32\linewidth]{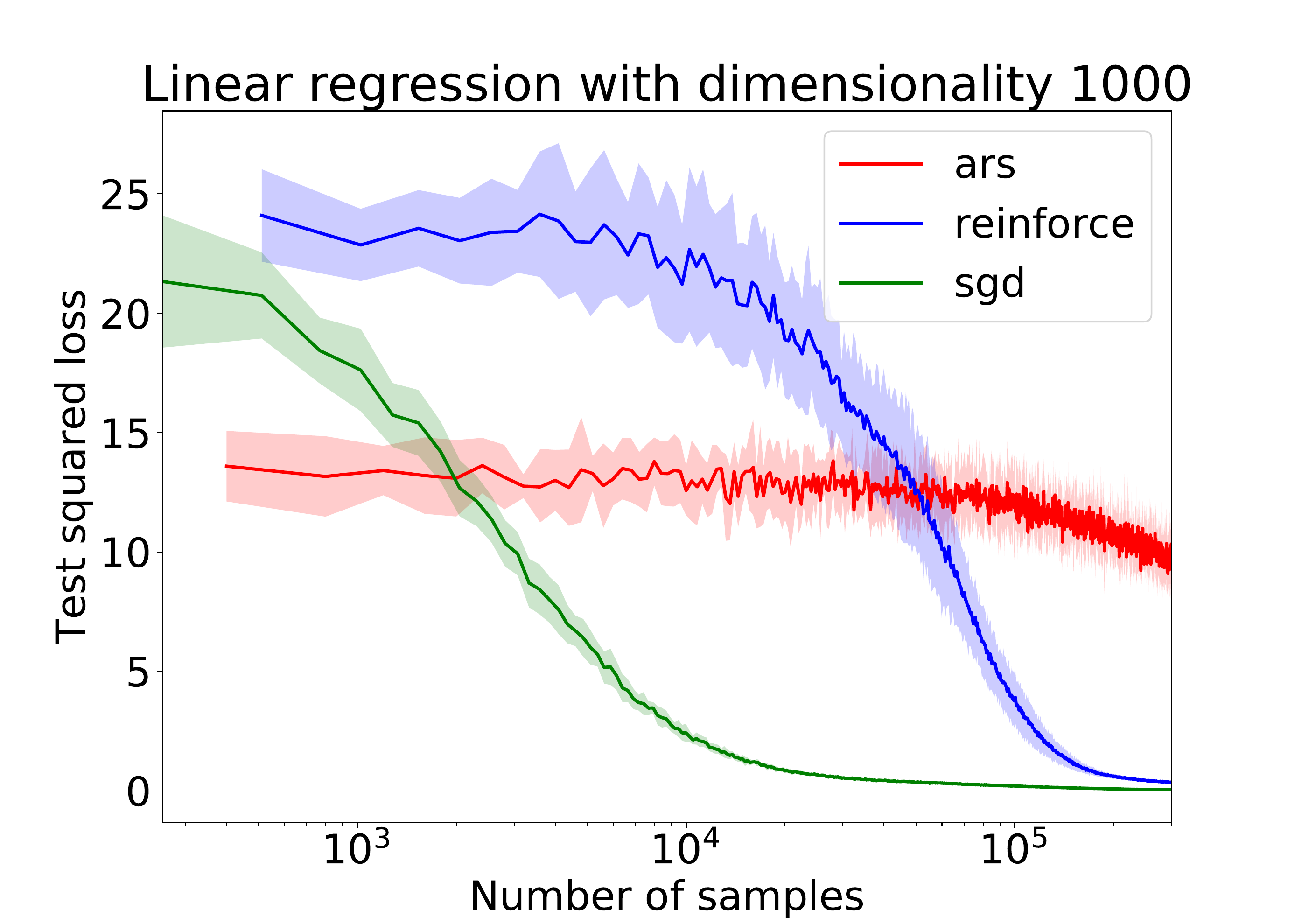}\label{fig:lin1000}}\end{subfigure}
  \caption{Linear Regression Experiments with varying input
    dimensionality}
  \label{fig:linear}
\end{figure*}
Figure \ref{fig:linear} shows the learning curves for standard full-information
supervised learning approaches with full access to the square loss (SGD \& Newton), REINFORCE, natural REINFORCE
\citep{kakade2002natural}, and ARS as we increase the 
input dimensionality, and hence parametric dimensionality $d$. 
Note that we have not included natural REINFORCE and Newton method in Figure \ref{fig:lin1000} as extensive hyperparameter search for these methods is computationally expensive in such high dimensionality settings.
The learning curves in Figure \ref{fig:linear}
match our expectations, and show that action space policy
search methods do not degrade as parametric dimensionality increases
whereas parameter space methods do. Moreover, action space methods lie
between the curves of supervised learning and parameter space
methods as they take advantage of the Jacobian of the policy and
learn  more quickly than parameter space methods.

\subsection{Multi-Step Control}
\label{sec:multi-step-control}

The above experiments provide insights on the dependence of
policy search methods on parametric
dimensionality $d$. We now shift our focus on to the dependence on horizon
length $H$.
In this set of experiments, we extend the time horizon and test the convergence rate of
policy search methods for multi-step control. The objective is to
minimize the sum of costs incurred over a horizon $H$, i.e. $J(\theta)
= \mathbb{E}[\sum_{t=1}^T c(s_t, a_t)]$. According to our analysis, we
expect action space policy search methods to have a dependence
on the horizon length $H$.

\begin{figure*}[t]
  \centering
  \begin{subfigure}{\includegraphics[width=0.32\linewidth]{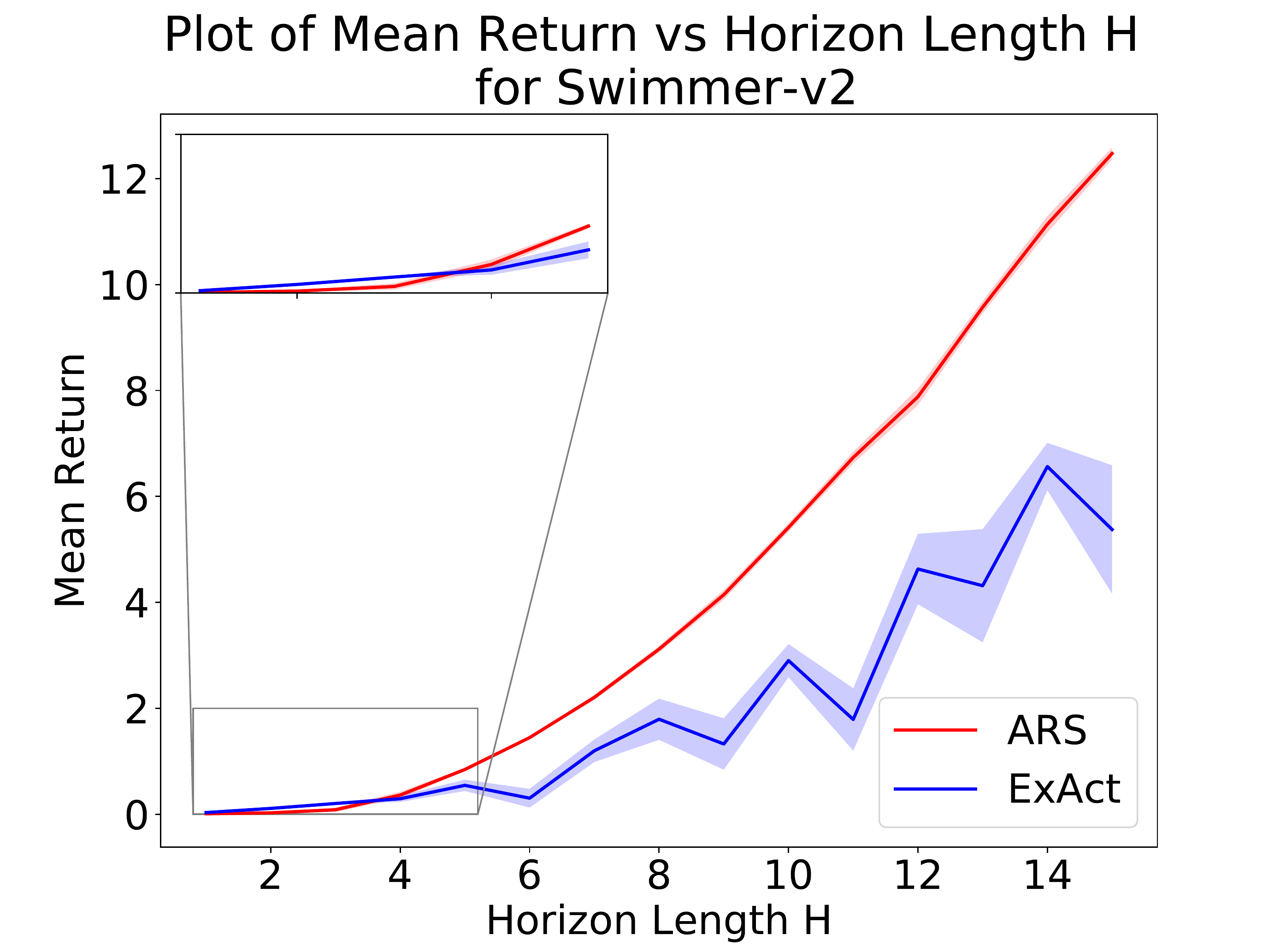}\label{fig:swimmer}}\end{subfigure}
  \begin{subfigure}{\includegraphics[width=0.32\linewidth]{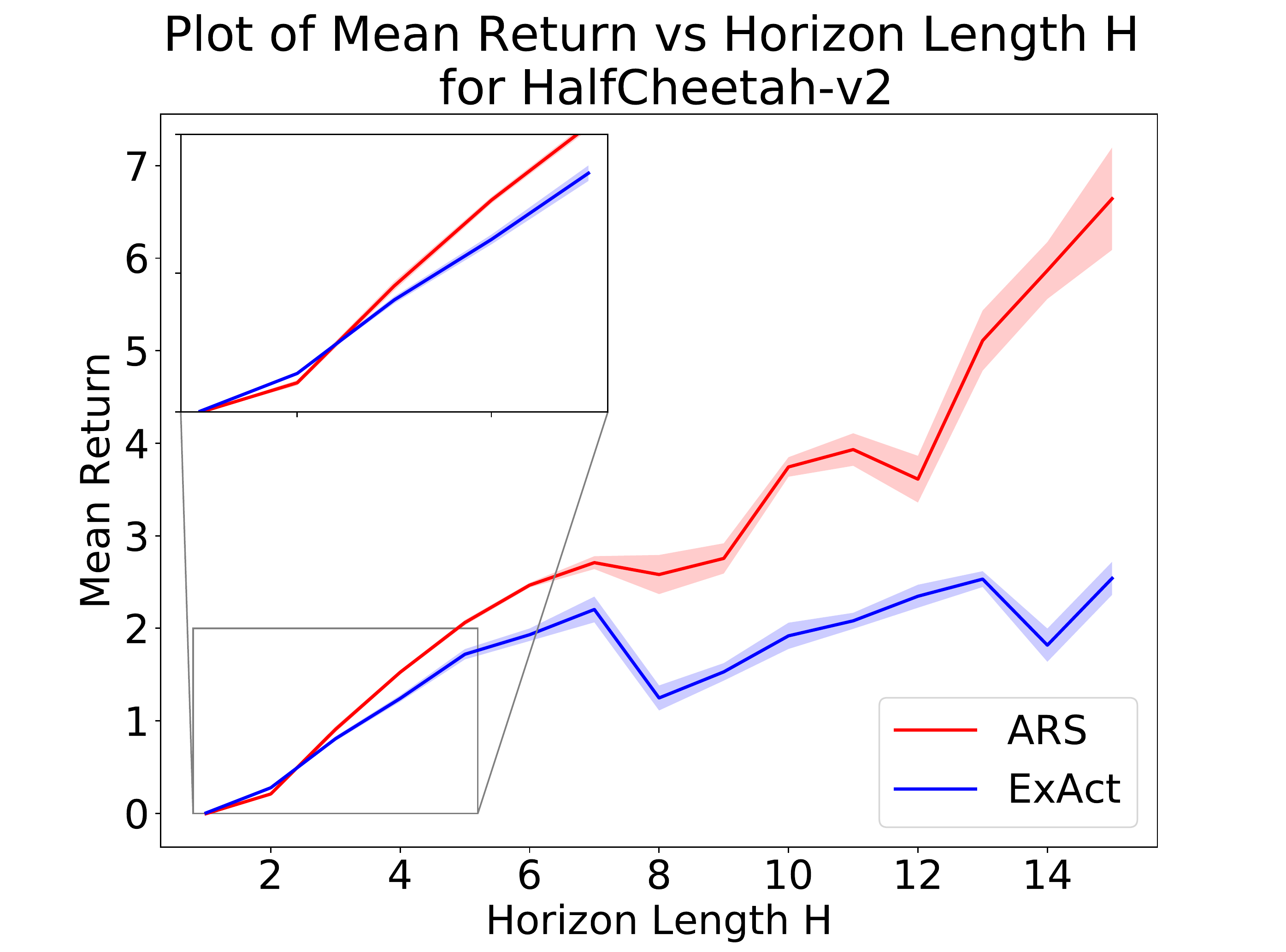}\label{fig:halfcheetah}}\end{subfigure}
  \begin{subfigure}{\includegraphics[width=0.32\linewidth]{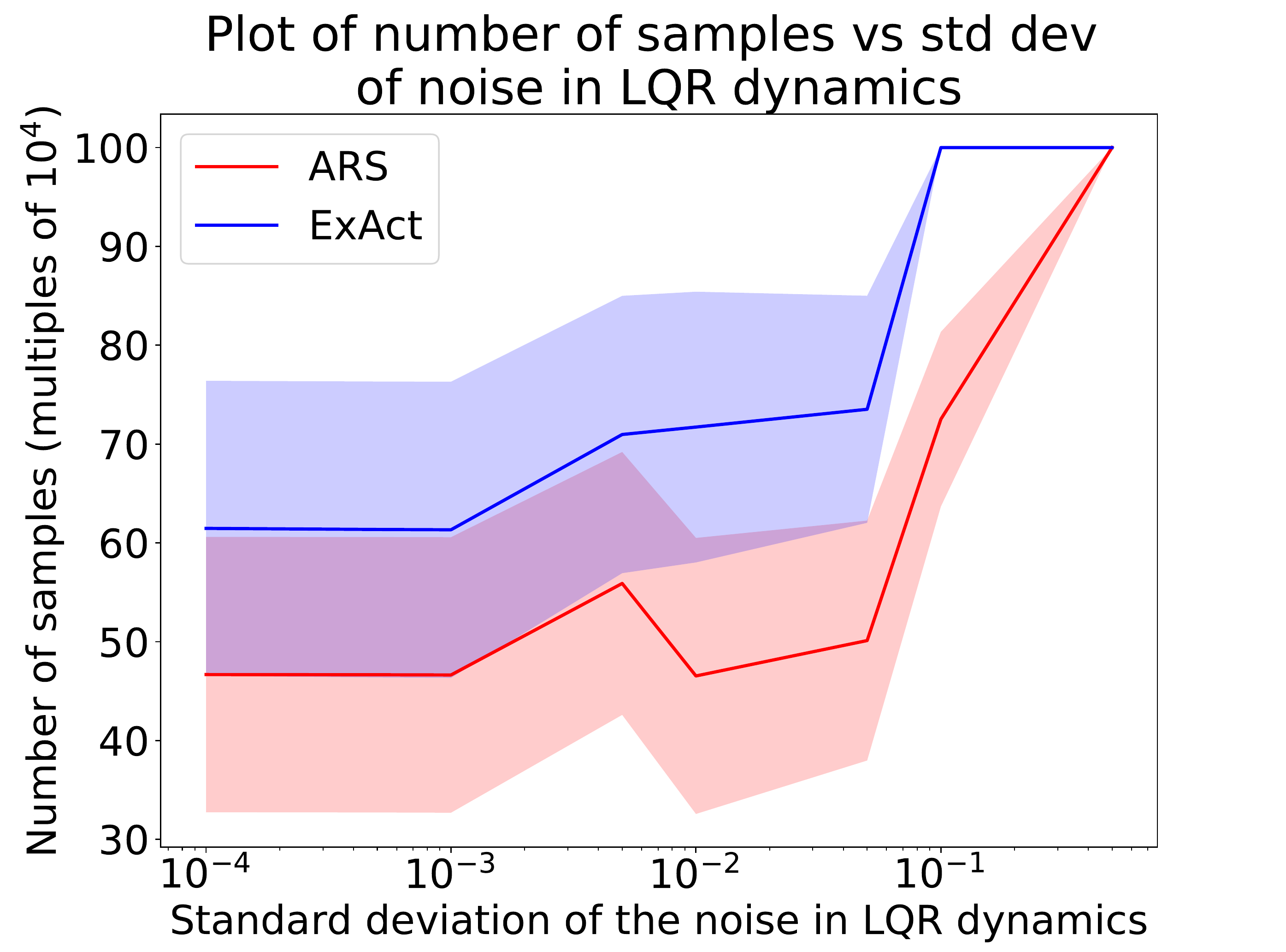}\label{fig:lqr}}\end{subfigure}
  \caption{Multi-step Control. Figures \ref{fig:swimmer} and \ref{fig:halfcheetah} show performance of different methods as horizon length varies. Figure \ref{fig:lqr} shows number of samples needed to reach close to a stationary point as noise in dynamics varies}
  \label{fig:multistep}
\end{figure*}



We test ARS and ExAct on two popular continuous control simulated
benchmark tasks in OpenAI gym \citep{openaigym}: Swimmer and
HalfCheetah. We chose these two environments as they allow you to vary
the horizon length $H$ without terminating the task early. For both
tasks, we use linear policies as they have been shown to be very
effective in \citep{mania2018simple, rajeswaran2017towards}. Swimmer
has an observation space dimensionality of $d = 8$ and a continuous
action space of dimensionality $p = 2$. Similarly, for HalfCheetah
$d=17$ and $p=6$. Figures \ref{fig:swimmer} and \ref{fig:halfcheetah} show the
performance of both approaches in terms of the mean return $J(\theta)$ (expected sum of
rewards) they obtain as the horizon length $H$ varies. 
Note that both
approaches are given access to the same number of samples $10^4 \times
H$ from the environments for each horizon length $H$. In the regime of
short horizon lengths, action space methods are better than parameter
space methods as they do not have a dependence on parametric
complexity $d$. However, as the horizon length increases, parameter
space methods start outperforming action space methods handily as they
do not have an explicit dependence on the horizon length, as pointed out
by our analysis. We have observed the same trend of parameter space
methods handily outperforming action space methods as far as $H =
200$ and expect this trend to continue beyond. This empirical insight
combined with our analysis presented in 
Sections \ref{sec:parameter_space}, \ref{sec:action_space} explains
why ARS, a simple parameter space search method, outperformed
state-of-the-art actor critic action space search methods in
\citep{mania2018simple} on OpenAI gym benchmarks where the horizon
length $H$ is typically as high as $1000$. 
\paragraph{Effect of environment stochasticity}
\label{sec:effect-stoch-param}

In this final set of experiments, we set out to understand the effect
of stochasticity in environment dynamics on the performance of
policy search methods. As our analysis in Sections
\ref{sec:parameter_space} and \ref{sec:action_space} points out, the
stochasticity of the environment plays an important role in
controlling the variance of our gradient estimates in zeroth order
optimization procedures. To empirically observe this, we use a
stochastic LQR environment where we have access to the true cost
function $c$ and hence, can compute the gradient $\nabla_\theta
J(\theta)$ exactly. Given access to such information, we vary the
standard deviation $\sigma$ of the noise in LQR dynamics and observe the number of samples needed for
ARS to reach $\theta$ such that
$\|\nabla_\theta J(\theta)\|_2^2 \leq 0.05$.
Figure \ref{fig:lqr} presents the number of samples needed to reach
close to a stationary point of $J(\theta)$ as the standard deviation
of noise in LQR dynamics varies. 
Note that we limit the maximum number of
samples to $10^6$ for each run. The results match our expectations from
the analysis, where we observed
that as the stochasticity of the environment increases, convergence
rate of policy search methods slows down.

\section{Conclusion}
\label{sec:conclusion}

Parameter space exploration via black-box optimization methods have
often been shown to outperform sophisticated action space exploration approaches for the reinforcement learning problem. Our work highlights the major difference between parameter and action space exploration methods: the latter leverages 
Jacobian of the parameterized policy. This allows sample complexity of action space exploration methods to be independent of parameter space dimensionality and only dependent on the dimensionality of action space and horizon length. For domains where the action space dimensionality and horizon length are small and the dimensionality of parameter space is large,
we conclude that exploration in action space should be preferred. On
the other hand, for long horizon control problems with low dimensional
policy parameterization, exploration in parameter space will
outperform exploration in action space.

\paragraph{Acknowledgements}
The authors would like to thank the anonymous reviewers for their
useful comments, the entire LairLab for stimulating discussions and
Ben Recht for his interesting blog posts.


\bibliography{ref}

\appendix

\newpage
\onecolumn
\section{Proof of Theorem~\ref{thm:online_linear_regression}}
\label{sec:proofs_bandit}
\begin{proof}[Proof of Theorem~\ref{thm:online_linear_regression}]

To prove Eq.~\ref{eq:random_para} for Alg.~\ref{alg:random_search_OLR}, we use the proof techniques from \cite{flaxman2005online}. The proof is more simpler than the one in \cite{flaxman2005online} as we do not have to deal with shrinking and reshaping the predictor set ${\Theta}$.

Denote $u\sim \mathbb{B}_b$ as uniformly sampling $u$ from a $b$-dim unit ball, $u\sim\mathbb{S}_b$ as uniformly sampling $u$ from the $b$-dim unit sphere, and $\delta \in (0,1)$. Consider the loss function $\hat{c}_i(w_i) = \mathbb{E}_{v\sim \mathbb{B}_b}[c_i(\theta_i + \delta v)]$, which is a smoothed version of $c_i(w_i)$. It is shown in \cite{flaxman2005online} that the gradient of $\hat{c}_i$ with respect to $\theta$ is:
\begin{align*}
   &\nabla_{\theta}\hat{c}_i(\theta)|_{\theta = \theta_i} \\
   &= \frac{b}{\delta} \mathbb{E}_{u\sim\mathbb{S}_b}[c_i(\theta_i +\delta u)u]\\
   &= \frac{b}{\delta}\mathbb{E}_{u\sim \mathbb{S}_b}[((\theta_i +\delta u)^T s_i - a_i)^2 u].
\end{align*} Hence, the descent direction we take in Alg.~\ref{alg:random_search_OLR} is actually an unbiased estimate of $\nabla_{\theta}\hat{c}_i(\theta)|_{\theta=\theta_i}$. So Alg.~\ref{alg:random_search_OLR} can be considered as running OGD with an unbiased estimate of gradient on the sequence of loss $\hat{c}_i(\theta_i)$. It is not hard to show that for an unbiased estimate of $\nabla_{\theta}\hat{c}_i(\theta)|_{\theta=\theta_i}$ = $\frac{b}{\delta} ((\theta_i + \delta u)^T s_i - a_i)^2 u$, the norm is bounded as $b(C^2 + C_{s}^2)/\delta$. Now we can directly applying Lemma 3.1 from \cite{flaxman2005online}, to get:
\begin{align}
\label{eq:regret_on_surrogate}
   \mathbb{E}\left[\sum_{i=1}^T \hat{c}_i(\theta_i)\right] - \min_{\theta^\star\in{\Theta}}\sum_{i=1}^T \hat{c}_i(\theta^\star) \leq \frac{C_{\theta}b(C^2+C_{s}^2)}{\delta}\sqrt{T}. 
\end{align} We can bound the difference between $\hat{c}_i(\theta)$ and ${c}_i(\theta)$ using the Lipschitiz continuous property of $c_i$:
\begin{align}
|\hat{c}_i(\theta) - c_i(\theta) | & = |\mathbb{E}_{v\sim \mathbb{B}_b}[c_i(\theta+\delta v) - c_i(\theta)]| \nonumber\\
&\leq \mathbb{E}_{v\sim \mathbb{B}_b}[|c_i(\theta+\delta v) - c_i(\theta)|] \leq L\delta. 
\end{align} Substitute the above inequality back to Eq.~\ref{eq:regret_on_surrogate}, rearrange terms, we get:
\begin{align}
&\mathbb{E}\left[ \sum_{i=1}^T c_i(\theta_i)  \right]  - \min_{\theta^\star\in{\Theta}} \sum_{i=1}^T c_i(w^\star)\leq \frac{C_{\theta}b(C^2+C_{s}^2)}{\delta}\sqrt{T} + 2LT\delta.
\end{align} By setting $\delta = T^{-0.25}\sqrt{\frac{C_{\theta}b(C^2+C_{s}^2)}{2L}}$, we get:
\begin{align*}
   &\mathbb{E}\left[ \sum_{i=1}^T c_i(\theta_i)  \right]  - \min_{w^\star\in{\Theta}} \sum_{i=1}^T c_i(w^\star) \leq \sqrt{C_{\theta}b(C^2+C_{s}^2)L} T^{3/4}. 
\end{align*}

To prove Eq.~\ref{eq:random_action} for Alg.~\ref{alg:random_search_action}, we follow the similar strategy in the proof of Alg.~\ref{alg:random_search_OLR}.

Denote $\epsilon \sim [-1,1]$ as uniformly sampling $\epsilon$ from the interval $[-1,1]$, $e\sim \{-1,1\}$ as uniformly sampling $e$ from the set containing $-1$ and $1$. Consider the loss function $\tilde{c}_i(\theta) = \mathbb{E}_{\epsilon\sim [-1,1]}[(\theta^T s_i + \delta \epsilon - a_i)^2]$. One can show that the gradient of $\tilde{c}_i(\theta)$ with respect to $\theta$ is:
\begin{align}
    \nabla_{\theta}\tilde{c}_i(\theta) = \frac{1}{\delta}\mathbb{E}_{e\sim \{-1,1\}}[e(\theta^{\top} s_i + \delta e - a_i)^2 s_i].
\end{align} As we can see that the descent direction we take in Alg.~\ref{alg:random_search_action} is actually an unbiased estimate of $\nabla_{\theta}\tilde{c}_i(\theta)|_{\theta=\theta_i}$. Hence Alg.~\ref{alg:random_search_action} can be considered as running OGD with unbiased estimates of gradients on the sequence of loss functions $\tilde{c}_i(\theta)$. For an unbiased estimate of the gradient, $\frac{1}{\delta} e(\theta_i^{\top} s_i +\delta e - a_i)^2 s_i$, its norm is bounded as $(C^2 + 1)C_{s}/\delta$. Note that different from Alg.~\ref{alg:random_search_OLR}, here the maximum norm of the unbiased gradient \emph{is independent of feature dimension $b$}. Now we apply Lemma 3.1 from \cite{flaxman2005online} on $\tilde{c}_i$, to get:
\begin{align}
\label{eq:tilde_random_action}
    \mathbb{E}\left[ \sum_{i=1}^T \tilde{c}_i(\theta_i)\right] - \min_{\theta^\star\in{\Theta}}\sum_{i=1}^T \tilde{c}_i(\theta^*) \leq \frac{C_{\theta}(C^2 + 1)C_{s}}{\delta}\sqrt{T}.
\end{align}
Again we can bound the difference between $\tilde{c}_i(\theta)$ and $c_i(\theta)$ for any $\theta$ using the fact that $(\hat{a}_i - a_i)^2$ is Lipschitz continuous with respect to prediction $\hat{a}_i$ with Lipschitz constant $C$: 
\begin{align}
    |\tilde{c}_i(\theta) - c_i(\theta)| &= |\mathbb{E}_{\epsilon\sim [-1,1]} [(\theta^{\top} s_i + \delta\epsilon - a_i)^2 - (\theta^{\top} s_i - a_i)^2]|  \nonumber\\
    &\leq \mathbb{E}_{\epsilon\sim [-1,-1]}[C\delta |\epsilon|] \leq C\delta.
\end{align} Substitute the above inequality back to Eq.~\ref{eq:tilde_random_action}, rearrange terms:
\begin{align*}
    &\mathbb{E}\left[\sum_{i=1}^T \tilde{c}_i(\theta_i)\right] - \min_{\theta^\star\in{\Theta}}\sum_{i=1}^T \tilde{c}_i(\theta^*)\leq \frac{C_{\theta}(C^2+1)C_{s}}{\delta}\sqrt{T} + 2C\delta T. 
\end{align*}
Set $\delta = T^{-0.25}\sqrt{\frac{C_{\theta}(C^2+1)C_{s}}{2C}}$, we get:
\begin{align*}
  &\mathbb{E}\left[\sum_{i=1}^T \tilde{c}_i(\theta_i)\right] - \min_{\theta^*\in{\Theta}}\sum_{i=1}^T \tilde{c}_i(\theta^*)\leq \sqrt{C_{\theta}(C^2+1)C_{s}C}T^{3/4}.  
\end{align*}
\end{proof}

\section{Proof of Theorem~\ref{theorem:parameter-convergence}}
\label{sec:proofs_RL}

We first present some useful lemmas below.

Consider the smoothed objective given by $\hat{J}(\theta) =
\mathbb{E}_{v \sim \mathbb{B}_d}[J(\theta + \delta v)]$ where
$\mathbb{B}_d$ is the unit ball in $d$ dimensions and $\delta$ is a
positive constant. Using the assumptions stated in Section
\ref{sec:assumptions_parameter}, we obtain the following useful lemma:
\begin{lemma}
  \label{lemma:grad-diff-parameter}
  If the objective $J(\theta)$ satisfies the assumptions in Section
  \ref{sec:assumptions_parameter} and the smoothed objective
  $\hat{J}(\theta)$
  is as given above,
  then we have that
  \begin{enumerate}
  \item $\hat{J}(\theta)$ is also $G$-Lipschitz and $L$-smooth
  \item For all $\theta \in \mathbb{R}^d$, $\|\nabla_\theta J(\theta)
    - \nabla_\theta \hat{J}(\theta)\| \leq L\delta$
  \end{enumerate}
\end{lemma}

\begin{proof}[Proof of Lemma \ref{lemma:grad-diff-parameter}]
  Consider for any $\theta_1, \theta_2 \in \mathbb{R}^d$,
\begin{align*}
    |\hat{J}(\theta_1) - \hat{J}(\theta_2)| &= |\mathbb{E}_{v \sim \mathbb{B}_d}[J(\theta_1+\delta v) - J(\theta_2 + \delta v)]| \nonumber \\
    &\leq \mathbb{E}_{v \sim \mathbb{B}_d}[|J(\theta_1+\delta v) - J(\theta_2 + \delta v)|] \nonumber \\
    &\leq \mathbb{E}_{v \sim \mathbb{B}_d}[G\|\theta_1 - \theta_2\|] \nonumber \\
    &= G\|\theta_1 - \theta_2\|
\end{align*}
The above inequalities are due to the fact that expectation of absolute value is greater than absolute value of expectation, and the $G$-lipschitz assumption on $J(\theta)$. Thus, the smoothened loss function $\hat{J}(\theta)$ is also $G$-lipschitz. Similarly consider,
\begin{align*}
  \|\nabla_\theta\hat{J}&(\theta_1) - \nabla_\theta\hat{J}(\theta_2)\| \\
  &= \|\nabla_\theta \mathbb{E}_{v \sim \mathbb{B}_d}[J(\theta_1 + \delta v)] - \nabla_\theta \mathbb{E}_{v \sim \mathbb{B}_d}[J(\theta_2+\delta v)]\| \nonumber \\
    &= \|\mathbb{E}_{v \sim \mathbb{B}_d}[\nabla_\theta J(\theta_1+\delta v)
      - \nabla_\theta J(\theta_2 + \delta v)]\| \nonumber \\
    &\leq \mathbb{E}_{v \sim \mathbb{B}_d}[\|\nabla_\theta J(\theta_1+\delta
      v) - \nabla_\theta J(\theta_2 + \delta v)\|] \nonumber \\
    &\leq \mathbb{E}_{v \sim \mathbb{B}_d}[L\|\theta_1 - \theta_2\|] \nonumber \\
    &= L\|\theta_1 - \theta_2\|
\end{align*}
The above inequalities are due to the fact that expectation of norm is
greater than norm of expectation, and the $L$-smoothness assumption on
$J(\theta_1)$. We interchange the expectation and derivative using the
assumptions on $J(\theta_1)$ and the dominated convergence
theorem. Thus, the smoothened loss function $\hat{J}(\theta_1)$ is
also $L$-smooth.

We know,
  \begin{align*}
    \nabla_\theta \hat{J}(\theta) &= \nabla_\theta\mathbb{E}_{v \sim \mathbb{B}_d}[J(\theta +
                             \delta v)] \nonumber \\
    &= \mathbb{E}_{v \sim \mathbb{B}_d}[\nabla_\theta J(\theta + \delta v)]
  \end{align*}
  Note that the expectation and derivative can be interchanged using
  the dominated convergence theorem. Hence, we have
  \begin{align*}
    \|\nabla_\theta \hat{J}(\theta) - \nabla_\theta J(\theta)\| &= \|\mathbb{E}_{u \sim
                                                  \mathbb{B}_d}[\nabla_\theta
                                                  J(\theta + \delta v)]
                                                  - \nabla_\theta J(\theta)\|
                                                  \nonumber \\
                                                &\leq \mathbb{E}_{u \sim
                                                  \mathbb{B}_d}\|\nabla_\theta
                                                  J(\theta + \delta v) -
                                                  \nabla_\theta J(\theta)\|
                                                  \nonumber \\
                                                &\leq \mathbb{E}_{u
                                                  \sim \mathbb{B}_d}[L
                                                  ||\delta v||]
                                                  \nonumber \\
                                                &\leq L \delta
  \end{align*}
\end{proof}

The above lemma will be very useful later when we try to relate the
convergence rate for the smoothed objective and the true objective. It is shown in
\citep{flaxman2005online, agarwal2010optimal} that the gradient estimate $g_i$ is an
unbiased estimator of the gradient $\nabla_\theta
\hat{J}(\theta_i)$. Hence, Algorithm \ref{alg:random_search_parameter}
is performing SGD on the smoothed objective $\hat{J}(\theta)$. Using
this insight, we can use the convergence rate of SGD for nonconvex
functions to stationary points from \citep{ghadimi2013stochastic} which is given as
follows
\begin{lemma}[\citep{ghadimi2013stochastic}]
  \label{lemma:sgd-parameter}
  Consider running SGD on the objective $\hat{J}(\theta)$ that is
  $L$-smooth and $G$-Lipschitz for $T$ steps. Fix initial solution
  $\theta_0$ and denote $\Delta_0 = \hat{J}(\theta_0) -
  \hat{J}(\theta^*)$ where $\theta^*$ is the point at which
  $\hat{J}(\theta)$ attains global minimum. Also, assume that the
  gradient estimate $g_i$ is unbiased and has a bounded variance,
  i.e. for all $i$, $\mathbb{E}_i[\|g_i - \nabla_\theta
  \hat{J}(\theta_i)\|_2^2] \leq V \in \mathbb{R}^+$ where
  $\mathbb{E}_i$ denotes expectation with randomness only at iteration
  $i$ conditioned on history upto iteration $i-1$. Then we have,
  \begin{equation}
    \frac{1}{T} \sum_{i=1}^T \mathbb{E}\|\nabla_\theta
    \hat{J}(\theta_i)\|_2^2 \leq \frac{2\sqrt{2\Delta_0L(V+G^2)}}{\sqrt{T}}
  \end{equation}
\end{lemma}
For completeness, we include a proof of the above lemma below. 
\begin{proof}[Proof of Lemma \ref{lemma:sgd-parameter}]
  Denote $\xi_i = g_i - \nabla_\theta {\hat{J}}(\theta_i)$.  Note that $\mathbb{E}_{i} [\xi_i] =
0$ since the stochastic gradient $g_i$ is unbiased.
From  $\theta_{i+1} = \theta_i - \alpha g_i$, we have:
\begin{align*}
  \hat{J}(\theta_{i+1}) & = \hat{J}(\theta_{i} - \alpha g_i)\\
  &\leq\hat{J}(\theta_i) - \nabla_\theta \hat{J}(\theta_i)^{\top} (\alpha g_i) + \frac{L\alpha^2}{2}\| g_i\|_2^2  \\
    & = \hat{J}(\theta_i) - \alpha \nabla_\theta \hat{J}(\theta_i)^{\top} g_i + \frac{L\alpha^2}{2} \|\xi_i + \nabla_\theta \hat{J}(\theta_i)\|^2_2 \\
    & = \hat{J}(\theta_i) - \alpha \nabla_\theta
      \hat{J}(\theta_i)^{\top} g_i + \frac{L\alpha^2}{2}(\|\xi_i\|_2^2
  + 2\xi_i^{\top}\nabla_\theta \hat{J}(\theta_i) + \|\nabla_\theta \hat{J}(\theta_i)\|_2^2 )
\end{align*} The first inequality above is obtained since the loss
function $\hat{J}(\theta)$ is $L$-smooth. Adding $\mathbb{E}_i$ on both sides and using the fact that $\mathbb{E}_i [\xi_i] = 0$, we have:
\begin{align*}
    \mathbb{E}_i [\hat{J}(\theta_{i+1})] &= \hat{J}(\theta_i) - \alpha
                                           \|\nabla_\theta
                                           \hat{J}(\theta_i)\|_2^2  +\frac{L\alpha^2}{2}\left( \mathbb{E}_i [\|\xi_i\|_2^2] + \|\nabla_\theta \hat{J}(\theta_i)\|_2^2  \right)  \\
    &\leq \hat{J}(\theta_i) - \alpha \|\nabla_\theta
      \hat{J}(\theta_i)\|_2^2 + \frac{L\alpha^2}{2}\left( \mathbb{E}_i [\|\xi_i\|_2^2] + G^2  \right) 
\end{align*} 
where the inequality is due to the lipschitz assumption. Rearranging terms, we get:
\begin{align*}
    \alpha\|\nabla_\theta \hat{J}(\theta_i)\|_2^2 &= \hat{J}(\theta_i)
      - \mathbb{E}_i [\hat{J}(\theta_{i+1})] + \frac{L\alpha^2}{2} (\mathbb{E}_i [\|\xi_i\|_2^2] + G^2) \\
    & \leq \hat{J}(\theta_i) - \mathbb{E}_i[ \hat{J}(\theta_{i+1})] + \frac{L\alpha^2}{2} (V + G^2) 
\end{align*}
Sum over from time step $1$ to $T$, we get:
\begin{align*}
    \alpha \sum_{t=1}^T \mathbb{E}\|\nabla_\theta
  \hat{J}(\theta_i)\|_2^2 &\leq \mathbb{E} [\hat{J}(\theta_0) -
                            \hat{J}(\theta_T)] + \frac{LT\alpha^2}{2}(V+G^2)
\end{align*} Divide $\alpha$  on both sides, we get:
\begin{align*}
    \sum_{t=1}^{T} \mathbb{E}&\|\nabla_\theta \hat{J}(\theta_i)\|_2^2 \leq \frac{1}{\alpha} \mathbb{E}[\hat{J}(\theta_0) - \hat{J}(\theta_T)] + {LT\alpha} (V+G^2) \\
    & \leq \frac{1}{\alpha} \mathbb{E}[\hat{J}(\theta_0) - \hat{J}(\theta^*)] + {LT\alpha} (V+G^2)  \\
    & = \frac{1}{\alpha} \Delta_0 + {LT\alpha} (V+G^2)  \\
    & \leq \sqrt{\frac{\Delta_0LT(V+G^2)}{2}} + \sqrt{2\Delta_0
      LT(V+G^2)} \\
      &\leq 2\sqrt{2\Delta_0 LT(V+G^2)}
\end{align*} with $\alpha = \sqrt{\frac{2\Delta_0}{LT(V+G^2)}}$.
Hence, we have:
\begin{align*}
    \frac{1}{T}\sum_{t=1}^T\mathbb{E}\|\nabla_\theta \hat{J}(\theta_i)\|_2^2 \leq \frac{2\sqrt{2\Delta_0 L(V+G^2)}}{\sqrt{T}}
\end{align*}
\end{proof}

The above lemma is useful as it gives us the following result:
\begin{align}
  \label{eq:stationary-point}
  \min_{1 \leq i \leq T} \mathbb{E}\|\nabla_\theta\hat{J}(\theta_i)\|_2^2 &\leq \frac{1}{T} \sum_{i=1}^T \mathbb{E}\|\nabla_\theta
                                                               \hat{J}(\theta_i)\|_2^2
                                                               \nonumber
  \\
                                                             &\leq \frac{2\sqrt{2\Delta_0L(V+G^2)}}{\sqrt{T}}
\end{align}
since the minimum is always less than the average. We have then that
using SGD to minimize a nonconvex objective finds a $\theta_i$ that is
`almost' a stationary point in bounded number of steps provided the
stochastic gradient estimate has bounded variance.

We now show that the gradient estimate $g_i$ used in Algorithm
\ref{alg:random_search_parameter} indeed has a bounded variance. Observe that
the estimate $g_i$ in the algorithm is a two-point estimate, which
should have substantially less variance than one-point
estimates \citep{agarwal2010optimal}. However, the two evaluations, resulting in $J_i^+$ and
$J_i^-$, have different independent noise. This is due to the
fact that in policy search, stochasticity arises from the
environment and cannot be controlled and we cannot obtain the
significant variance reduction that is typical of two-point
estimators. The following lemma quantifies the bound on the variance of
gradient estimate $g_i$:
\begin{lemma}
\label{lemma:grad-variance-parameter}
  Consider a smoothed objective $\hat{J}(\theta) = \mathbb{E}_{v \sim
    \mathbb{B}_d}[J(\theta + \delta v)]$ where $\mathbb{B}_d$ is the
  unit ball in $d$ dimensions, $\delta > 0$ is a scalar and the true
  objective $J(\theta)$ is $G$-lipschitz. Given gradient estimate $g_i
  = \frac{d(J_i^+ - J_i^-)}{2\delta}u$ where $u$ is sampled uniformly
  from a unit sphere $\mathbb{S}_d$ in $d$ dimensions, $J^+_i =
  J(\theta_i + \delta u) + \eta^+_i$ and $J_i^- = J(\theta - \delta u)
  + \eta_i^-$ for zero mean random i.i.d noises $\eta_i^+, \eta_i^-$, we have
  \begin{equation}
    \label{eq:grad-variance-parameter}
    \mathbb{E}_i[\|g_i - \nabla_\theta \hat{J}(\theta_i)\|_2^2] \leq
    2d^2G^2 + 2\frac{d^2\sigma^2}{\delta^2}
  \end{equation}
  where {$\sigma^2$ is the variance of the random noise $\eta$.}
\end{lemma}

\begin{proof}[Proof of Lemma \ref{lemma:grad-variance-parameter}]
  From \cite{shamir2017optimal}, we know that $g_i$ is an unbiased estimate of the gradient of $\hat{J}(\theta_i)$, i.e. $\mathbb{E}_{u_i \sim \mathbb{S}_d}[g_i] = \nabla\hat{J}(\theta_i)$. Thus, we have
\begin{align*}
    \mathbb{E}_{u_i \sim \mathbb{S}_d}&\|g_i -
      \nabla\hat{J}(\theta_i)\|^2 \\
      &= \mathbb{E}_{u_i \sim \mathbb{S}_d}[\|g_i\|^2 + \|\nabla \hat{J}(\theta)_i\|^2 - 2g_i^T\nabla \hat{J}(\theta_i)] \\
    &= \mathbb{E}_{u_i \sim \mathbb{S}_d}\|g_i\|^2 + \|\nabla \hat{J}(\theta_i)\|^2 - 2\|\nabla \hat{J}(\theta_i)\|^2 \\
    &= \mathbb{E}_{u_i \sim \mathbb{S}_d}\|g_i\|^2 - \|\nabla \hat{J}(\theta_i)\|^2 \\
    &\leq \mathbb{E}_{u_i \sim \mathbb{S}_d}\|g_i\|^2 \\
    &= \frac{d^2}{4\delta^2}\mathbb{E}_{u_i \sim
      \mathbb{S}_d}\|(J(\theta_i + \delta u_i) - J(\theta_i - \delta
      u_i) + (\eta_i^+ - \eta_i^-))u_i\|^2 \\
  &\leq \frac{d^2}{2\delta^2}[\mathbb{E}_{u_i \sim \mathbb{S}_d}\|(J(\theta_i + \delta u_i) - J(\theta_i - \delta
    u_i)u_i\|_2^2 + \mathbb{E}_{u_i \sim \mathbb{S}_d}\|(\eta_i^+ -
    \eta_i^-))u_i\|^2] \\
    &\leq \frac{d^2}{2\delta^2}[\mathbb{E}_{u_i \sim \mathbb{S}_d}
      4G^2\delta^2 \|u_i\|^2 +  4\mathbb{E}_{u_i \sim
      \mathbb{S}_d}\|\eta_i^+\|_2^2 \|u_i\|_2^2]\\
    &= 2d^2G^2  + 2\frac{d^2\sigma^2}{\delta^2}
\end{align*}
where the second inequality is true as $\|a+b\|_2^2 \leq 2(\|a\|_2^2 +
\|b\|_2^2)$ and the last inequality is due to the Lipschitz assumption
on $J(\theta)$.
\end{proof}

We are ready to prove Theorem~\ref{theorem:parameter-convergence}. 
\begin{proof}[Proof of Theorem \ref{theorem:parameter-convergence}]
  Fix initial solution $\theta_0$ and denote $\Delta_0 =
  \hat{J}(\theta_0) - \hat{J}(\theta^*)$ where $\hat{J}(\theta)$ is
  the smoothed objective and $\theta^*$ is the point at which
  $\hat{J}(\theta)$  attains global minimum.
  Since the gradient estimate $g_i$ used in Algorithm
  \ref{alg:random_search_parameter} is an unbiased estimate of the
  gradient $\nabla_\theta \hat{J}(\theta_i)$, we know that Algorithm
  \ref{alg:random_search_parameter} performs SGD on the smoothed
  objective. Moreover, from Lemma \ref{lemma:grad-variance-parameter},
  we know that the variance of the gradient estimate $g_i$ is
  bounded. Hence, we can use Lemma \ref{lemma:sgd-parameter} on the
  smoothed objective $\hat{J}(\theta)$ to get
  \begin{align}
    \label{eq:sgd-parameter}
    \frac{1}{T} \sum_{i=1}^T \mathbb{E}\|\nabla_\theta
    \hat{J}(\theta_i)\|_2^2 \leq \frac{2\sqrt{2\Delta_0L(V+G^2)}}{\sqrt{T}}
  \end{align}
  where $V \leq 2d^2G^2 + 2\frac{d^2\sigma^2}{\delta^2}$ (from Lemma
  \ref{lemma:grad-variance-parameter}). We can relate $\nabla_\theta
  \hat{J}(\theta)$ and $\nabla_\theta J(\theta)$ - the quantity that
  we ultimately care about, as follows:
  \begin{align*}
    \frac{1}{T} &\sum_{i=1}^T \mathbb{E}\|\nabla_\theta
    J(\theta_i)\|_2^2\\
    &= \frac{1}{T} \sum_{i=1}^T
                        \mathbb{E}\|\nabla_\theta J(\theta_i) -
                        \nabla_\theta \hat{J}(\theta_i) +
      \nabla_\theta \hat{J}(\theta_i)\|_2^2 \\
    &\leq \frac{2}{T} \sum_{i=1}^T \mathbb{E}\|\nabla_\theta J(\theta_i) -
                        \nabla_\theta \hat{J}(\theta_i)\|_2^2 + \mathbb{E}\|\nabla_\theta \hat{J}(\theta_i)\|_2^2
  \end{align*}
  We can use Lemma \ref{lemma:grad-diff-parameter} to bound the first
  term and Equation \ref{eq:sgd-parameter} to bound the second
  term. Thus, we have
  \begin{align*}
    \frac{1}{T} \sum_{i=1}^T \mathbb{E}\|\nabla_\theta
    J(\theta_i)\|_2^2 \leq \frac{2}{T}[TL^2\delta^2 + 2\sqrt{2\Delta_0L(V+G^2)T}]
  \end{align*}
  Substituting the bound for $V$ from Lemma
  \ref{lemma:grad-variance-parameter}, using the inequality
  $\sqrt{a+b} \leq \sqrt{a} + \sqrt{b}$ for $a, b \in \mathbb{R}^+$,
  optimizing over $\delta$, and using $\Delta_0 \leq \Qbound$ we get
  \begin{equation*}
    \frac{1}{T} \sum_{i=1}^T \mathbb{E}\|\nabla_\theta
    J(\theta_i)\|_2^2 \leq \mathcal{O}(\Qbound^{\frac{1}{2}}dT^{\frac{-1}{2}} + \Qbound^{\frac{1}{3}}d^{\frac{2}{3}}T^{\frac{-1}{3}}\sigma)
  \end{equation*}
\end{proof}

\section{Proof of Theorem }

The bound on the bias of the gradient estimate is given
by the following lemma:
\begin{lemma}
  \label{lemma:bias-bound-action}
  If the assumptions in Section \ref{sec:assumptions_action} are
  satisfied, then for  the gradient estimate $g_i$ used in Algorithm
  \ref{alg:random_search_action} and the gradient of the objective
  $J(\theta)$ given in equation \ref{eq:dpg-gradient}, we have
  \begin{equation}
    \label{eq:bias-bound-action}
    \|\mathbb{E}[g_i] - \nabla_\theta J(\theta_i)\| \leq KUH\delta
  \end{equation}
\end{lemma}

\begin{proof}[Proof of Lemma \ref{lemma:bias-bound-action}]
  To prove that the bias is bounded, let's consider for any $i$
  \begin{align*}
      &\|\mathbb{E}[g_i] - \nabla_\theta J(\theta_i)\|_2 = \|\sum_{t=0}^{H-1} \mathbb{E}_{s_t\sim
        d_{\pi_{\theta_i}}^t}[\nabla_\theta \pi(\theta_i, s_t) \nabla_a (\mathbb{E}_{v \sim \mathbb{B}_p}
        Q_{\pi_{\theta_i}}^t(s_t, \pi(\theta_i, s_t) + \delta v) -
          Q_{\pi_{\theta_i}}^t(s_t, \pi(\theta_i, s_t)))]\|_2 \\
    &\leq \sum_{t=0}^{H-1} \mathbb{E}_{s_t\sim d_{\pi_{\theta_i}}^t, v
      \sim \mathbb{B}_p}\|\nabla_\theta \pi(\theta_i, s_t)\|_2 \|[\nabla_a
      Q_{\pi_{\theta_i}}^t(s_t, \pi(\theta_i, s_t) + \delta v) -
      \nabla_a Q_{\pi_{\theta_i}}^t(s_t, \pi(\theta_i, s_t))]\|_2 \\
    &\leq \sum_{t=0}^{H-1} KU\delta \mathbb{E}_{v \sim
      \mathbb{B}_p}\|v\|_2 \\
    &\leq KUH\delta
  \end{align*}
  The first inequality above is obtained by using the fact that
  $\|\mathbb{E}[X]\|_2 \leq \mathbb{E}\|X\|_2$, and the second
  inequality using the $K$-lipschitz assumption on $\pi(\theta, s)$
  and $U$-smooth assumption on $Q_{\pi_\theta}^t(s, a)$ in $a$. Also,
  observe that we interchanged the derivative and expectation above by
  using the assumptions on $Q_{\pi_\theta}^t$ as stated in Section
  \ref{sec:assumptions_action}. 
\end{proof}

We will now show that the gradient estimate $g_i$ used in Algorithm
\ref{alg:random_search_action} has a bounded variance. Note that the
gradient estimate constructed in Algorithm
\ref{alg:random_search_action} is a one-point estimate, unlike policy
search in parameter space where we had a two-point estimate. Thus, the variance would be higher and the bound
on the variance of such a one-point estimate is given below
\begin{lemma}
  \label{lemma:grad-variance-action}
  Given a gradient estimate $g_i$ as shown in Algorithm
  \ref{alg:random_search_action}, the variance of the estimate can be
  bounded as
  \begin{equation}
    \label{eq:grad-variance-action}
    \mathbb{E}\|g_i - \mathbb{E}[g_i]\|_2^2 \leq
    \frac{2H^2p^2K^2}{\delta^2} ((\Qbound + W\delta)^2 + \sigma^2)
  \end{equation}
  where $\sigma^2$ is the variance of the random noise $\tilde{\eta}$.
\end{lemma}

\begin{proof}[Proof of Lemma \ref{lemma:grad-variance-action}]
  To bound the variance of the gradient estimate $g_i$ in Algorithm
  \ref{alg:random_search_action}, lets consider 
  \begin{align*}
    &\mathbb{E}_i\|g_i - \mathbb{E}[g_i]\|_2^2 = \mathbb{E}_i\|g_i\|_2^2 -
                                              \|\mathbb{E}_i[g_i]\|_2^2
                                            \leq \mathbb{E}_i\|g_i\|_2^2 \\
    &= \frac{H^2p^2}{\delta^2} \mathbb{E}_i\|\nabla_\theta
      \pi(\theta_i, s_t)
    (Q_{\pi_{\theta_i}}^t(s_t, \pi(\theta_i, s_t)
      + \delta u) + \tilde{\eta}_i) u\|_2^2 \\
    &\leq \frac{K^2p^2H^2}{\delta^2} \mathbb{E}_{i}\|Q_{\pi_{\theta_i}}^t(s_t, \pi(\theta_i,
      s_t) + \delta u)u + \tilde{\eta}_iu\|_2^2
  \end{align*}
  where $\mathbb{E}_i$ denotes expectation with respect to the
  randomness at iteration $i$ and the inequality is obtained using
  $K$-lipschitz assumption on $\pi(\theta, s)$. Note that we can
  express $Q_{\pi_{\theta_i}}^t(s_t, \pi(\theta_i, s_t) + \delta u)
  \leq Q_{\pi_{\theta_i}}^t(s_t, \pi(\theta_i, s_t)) + W\delta\|u\|_2
  \leq \Qbound + W\delta$ where we used the $W$-lipschitz assumption on
  $Q_{\pi_{\theta}}^t(s, a)$ in $a$ and that it is bounded everywhere
  by constant $\Qbound$. Thus, we have
  \begin{align*}
    &\mathbb{E}_i\|g_i - \mathbb{E}[g_i]\|_2^2 \\
    &\leq \frac{K^2p^2H^2}{\delta^2} \mathbb{E}_{i}\|(\Qbound + W\delta)u +
      \tilde{\eta}_iu\|_2^2 \\
    &\leq \frac{2K^2p^2H^2}{\delta^2}
      (\mathbb{E}_i\|(\Qbound+W\delta)u\|_2^2 +
      \mathbb{E}_i\|\tilde{\eta}_iu\|_2^2 \\
    &\leq \frac{2K^2p^2H^2}{\delta^2} ((\Qbound+W\delta)^2 + \sigma^2)
  \end{align*}
\end{proof}

We are now ready to prove theorem \ref{theorem:action-convergence}
\begin{proof}[Proof of Theorem \ref{theorem:action-convergence}]
  Fix initial solution $\theta_0$ and denote $\Delta_0 =
  J(\theta_0) - J(\theta^*)$ where $\theta^*$ is the point at which
  $J(\theta)$ attains global minimum. Denote $\xi_i = g_i - \mathbb{E}_i[g_i]$ and $\beta_i =
  \mathbb{E}_i[g_i] - \nabla_\theta J(\theta_i)$. From Lemma
  \ref{lemma:bias-bound-action}, we know $\|\beta_i\| \leq KUH\delta$
  and from lemma \ref{lemma:grad-variance-action}, we know
  $\mathbb{E}\|\xi_i\|_2^2 = V \leq \frac{2K^2p^2H^2}{\delta^2} ((\Qbound +
  W\delta)^2 + \sigma^2)$ and $\mathbb{E}_i[\xi_i] = 0$ from definition. From $\theta_{i+1} = \theta_i - \alpha g_i$
  we have:
  \begin{align*}
    J(\theta_{i+1}) &= J(\theta_i - \alpha g_i) \\
    &\leq J(\theta_i) - \alpha \nabla_\theta J(\theta_i)^Tg_i +
      \frac{L\alpha^2}{2}\|g_i\|_2^2 \\
    &= J(\theta_i) - \alpha \nabla_\theta J(\theta_i)^T g_i +
      \frac{L\alpha^2}{2}\|\xi_i + \mathbb{E}_i[g_i]\|_2^2 \\
                    &= J(\theta_i) - \alpha \nabla_\theta J(\theta_i)^T g_i +
      \frac{L\alpha^2}{2}(\|\mathbb{E}_i[g_i]\|_2^2 + \|\xi_i\|_2^2 + 2\mathbb{E}_i[g_i]^T\xi_i)
  \end{align*}
  Taking expectation on both sides with respect to randomness at
  iteration $i$, we have
  \begin{align*}
    &\mathbb{E}_i[J(\theta_{i+1})] = J(\theta_i) - \alpha\nabla_\theta
    J(\theta_i)^T\mathbb{E}_i[g_i] + \frac{L\alpha^2}{2}(\|\mathbb{E}_i [g_i]\|_2^2 +
    \mathbb{E}_i\|\xi_i\|_2^2 +
      2\mathbb{E}_i[g_i]^T\mathbb{E}_i[\xi_i]) \\
    &\leq J(\theta_i) - \alpha\nabla_\theta J(\theta_i)^T (\beta_i +
      \nabla_\theta J(\theta_i)) +\frac{L\alpha^2}{2}(\|\beta_i + \nabla_\theta
      J(\theta_i)\|_2^2 + V) \\
    &= J(\theta_i) - \alpha\|\nabla_\theta J(\theta_i)\|_2^2 +
      \frac{L\alpha^2}{2}(\|\nabla_\theta J(\theta_i)\|_2^2 + V +
      \|\beta_i\|_2^2) + (L\alpha^2 - \alpha)\nabla_\theta J(\theta_i)^T\beta_i \\
    &\leq J(\theta_i) - \alpha\|\nabla_\theta J(\theta_i)\|_2^2 +
      \frac{L\alpha^2}{2}(G^2 + V + K^2H^2U^2\delta^2) + (L\alpha^2 - \alpha)\nabla_\theta J(\theta_i)^T\beta_i \\
    &\leq J(\theta_i) - \alpha\|\nabla_\theta J(\theta_i)\|_2^2 +
      \frac{L\alpha^2}{2}(G^2 + V + K^2H^2U^2\delta^2) + (L\alpha^2 + \alpha)\|\nabla_\theta
      J(\theta_i)\|\|\beta_i\| \\
    &\leq J(\theta_i) - \alpha\|\nabla_\theta J(\theta_i)\|_2^2 +
      \frac{L\alpha^2}{2}(G^2 + V + K^2H^2U^2\delta^2) + (L\alpha^2 + \alpha)GKUH\delta
  \end{align*}
  Rearranging terms and summing over timestep $1$ to $T$, we get
  \begin{align*}
    &\alpha\sum_{i=1}^T \|\nabla_\theta J(\theta_i)\|_2^2 \leq J(\theta_0) -
      \mathbb{E}_T[J(\theta_T)] + \frac{LT\alpha^2}{2}(G^2 + V +
      K^2H^2U^2\delta^2) + (L\alpha^2 + \alpha)GKUHT\delta \\
    &\leq \Delta_0 + \frac{LT\alpha^2}{2}(G^2 + V +
      K^2H^2U^2\delta^2) + (L\alpha^2 + \alpha)GKUHT\delta \\
    &\sum_{i=1}^T \|\nabla_\theta J(\theta_i)\|_2^2 \leq
      \frac{\Delta_0}{\alpha} + \frac{LT\alpha}{2}(G^2 + V +
      K^2H^2U^2\delta^2) + (L\alpha + 1)GKUHT\delta \\
    &\leq \frac{\Delta_0}{\alpha} + \frac{LT\alpha}{2}(G^2 +
      K^2H^2U^2\delta^2 + 2GKUH\delta)+ GKUHT\delta + \frac{LT\alpha}{2}V \\
    &\leq \frac{\Delta_0}{\alpha} + \frac{LT\alpha}{2}(G+KHU\delta)^2
    + GKUHT\delta + 
      \frac{LT\alpha K^2p^2H^2}{\delta^2}((\Qbound + W\delta)^2 + \sigma^2) \\
      &\leq  \frac{\Delta_0}{\alpha} + LT\alpha(G^2 + K^2H^2U^2\delta^2) + GKUHT\delta + 2\frac{LT\alpha K^2p^2H^2}{\delta^2}(\Qbound^2 + W^2\delta^2 +\sigma^2)
  \end{align*}
  Using $\Delta_0 \leq \Qbound$ and optimizing over $\alpha$ and $\delta$, we get $\alpha = \mathcal{O}(\Qbound^{\frac{3}{4}}T^{-\frac{3}{4}}H^{-1}p^{-\frac{1}{2}}(\Qbound^2 + \sigma^2)^{-\frac{1}{4}})$ and $\delta = \mathcal{O}(T^{-\frac{1}{4}}p^{\frac{1}{2}}(\Qbound^2 + \sigma^2)^{\frac{1}{4}})$. This gives us
  \begin{equation}
      \frac{1}{T}\sum_{i=1}^T \|\nabla_\theta J(\theta_i)\|_2^2 \leq \mathcal{O}(T^{-\frac{1}{4}}Hp^{\frac{1}{2}}(\Qbound^3 + \sigma^2\Qbound)^{\frac{1}{4}})
  \end{equation}
  
\end{proof}

\section{Implementation Details}
\label{sec:impl-deta}

\subsection{One-step Control Experiments}
\label{sec:one-step-control-1}

\subsubsection{Tuning Hyperparameters for ARS}
We tune the hyperparameters for ARS \citep{mania2018simple} in both MNIST and linear regression experiments, by choosing a candidate set of values for each hyperparameter: stepsize, number of directions sampled, number of top directions chosen and the perturbation length along each direction. The candidate hyperparameter values are shown in Table \ref{tab:hyperparam}. 

\begin{table}[ht]
    \centering
    \begin{tabular}{|c|c|}
      \hline
      \textbf{Hyperparameter} & \textbf{Candidate Values}\\
    \hline
    Stepsize &  $0.001, 0.005, 0.01, 0.02, 0.03$\\
    \hline
    \# Directions &  $10, 50, 100, 200, 500$\\
    \hline
    \# Top Directions & $5, 10, 50, 100, 200$\\
    \hline
    Perturbation & $0.001, 0.005, 0.01, 0.02, 0.03$ \\
    \hline
    \end{tabular} 
    \caption{Candidate hyperparameters used for tuning in ARS  experiments}
    \label{tab:hyperparam}
\end{table}
           
We use the hyperparameters shown in Table \ref{tab:chosen-hyperparams} chosen through this tuning for each of the experiments in this work. The hyperparameters are chosen by averaging the test squared loss across three random seeds (different from the 10 random seeds used in actual experiments) and chosing the setting that has the least mean test squared loss after 100000 samples.

\begin{table}[ht]
    \centering
    \begin{tabular}{|c|c|c|c|c|}
    \hline
    \textbf{Experiment} & \textbf{Stepsize} & \textbf{\# Dir}. & \textbf{\# Top Dir.} & \textbf{Perturbation}\\
    \hline
    MNIST     &  0.02 & 50 & 20 & 0.03\\
    \hline
    LR $d=10$     & 0.03 & 10 & 10 & 0.03 \\
    \hline
    LR $d=100$ & 0.03 & 10 & 10 & 0.02 \\
    \hline
    LR $d=1000$ & 0.03 & 200 & 200 & 0.03 \\
    \hline
    \end{tabular}
    \caption{Hyperparameters chosen for ARS in each experiment. LR is short-hand for Linear Regression.}
    \label{tab:chosen-hyperparams}
\end{table}

\begin{table}[ht]
    \centering
    \begin{tabular}{|c|c|c|}
    \hline
    \textbf{Experiment}     &  \textbf{Learning Rate} & \textbf{Batch size}\\
    \hline
    MNIST     &  0.001 & 512\\
    \hline
    LR $d=10$ & 0.08 & 512\\
    \hline
    LR $d=100$ & 0.03 & 512\\
    \hline
    LR $d=1000$ & 0.01 & 512\\
    \hline
    \end{tabular}
    \caption{Learning rate and batch size used for REINFORCE experiments. We use an ADAM \citep{kingma2014adam} optimizer for these experiments.}
    \label{tab:hyperparam-reinforce}
\end{table}

\begin{table}[ht]
    \centering
    \begin{tabular}{|c|c|c|}
    \hline
    \textbf{Experiment}     &  \textbf{Learning Rate} & \textbf{Batch size}\\
    \hline
    LR $d=10$ & 2.0 & 512\\
    \hline
    LR $d=100$ & 2.0 & 512\\
    \hline
    \end{tabular}
    \caption{Learning rate and batch size used for Natural REINFORCE experiments. Note that we decay the learning rate after each batch by $\sqrt{T}$ where $T$ is the number of batches seen.}
    \label{tab:hyperparam-nreinforce}
\end{table}

\subsubsection{MNIST Experiments}
\label{sec:mnist-details}

The CNN architecture used is as shown in Figure \ref{fig:arch}\footnote{This figure is generated by adapting the code from \url{https://github.com/gwding/draw_convnet}}. The total number of parameters in this model is $d=21840$. For supervised learning, we use a cross-entropy loss on the softmax output with respect to the true label. To train this model, we use a batch size of 64 and a stochastic gradient descent (SGD) optimizer with learning rate of 0.01 and a momentum factor of 0.5. We evaluate the test accuracy of the model over all the $10000$ images in the MNIST test dataset. 

\begin{figure}[H]
    \centering
    \includegraphics[width=0.9\linewidth]{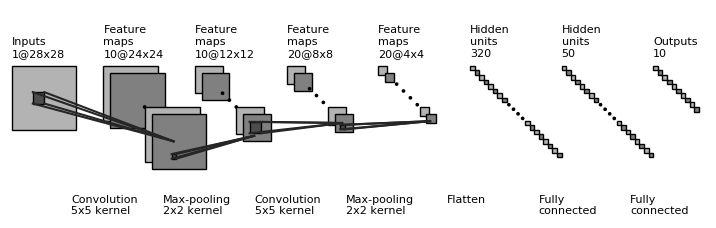}
    \caption{CNN architecture used for the MNIST experiments}
    \label{fig:arch}
\end{figure}

For REINFORCE, we use the same architecture as before. We train the model by sampling from the categorical distribution parameterized by the softmax output of the model and then computing a $\pm 1$ reward based on whether the model predicted the correct label. The loss function is the REINFORCE loss function given by,
\begin{equation}
    J(\theta) = \frac{1}{N} \sum_{i=1}^N r_i \log(\mathbb{P}(\hat y_i|x_i, \theta))
\end{equation}
where $\theta$ is the parameters of the model, $r_i$ is the reward obtained for example $i$, $\hat y_i$ is the predicted label for example $i$ and $x_i$ is the input feature vector for example $i$. The reward $r_i$ is given by $r_i = 2*\mathbb{I}[\hat y_i = y_i] - 1$, where $\mathbb{I}$ is the $0-1$ indicator function and $y_i$ is the true label for example $i$.

For ARS, we use the same architecture and reward function as before. The hyperparameters used are shown in Table \ref{tab:chosen-hyperparams} and we closely follow the algorithm outlined in \citep{mania2018simple}.

\subsubsection{Linear Regression Experiments}
\label{sec:linreg-details}

We generate training and test data for the linear regression experiments as follows: we sampled a random $d+1$ dimensional vector $w$ where $d$ is the input dimensionality. We also sampled a random $d \times d$ covariance matrix $C$. The training and test dataset consists of $d+1$ vectors $x$ whose first element is always $1$ (for the bias term) and the rest of the $d$ terms are sampled from a multivariate normal distribution with mean $\mathbf{0}$ and covariance matrix $C$. The target vectors $y$ are computed as $y = w^Tx + \epsilon$ where $\epsilon$ is sampled from a univariate normal distribution with mean $0$ and standard deviation $0.001$.

We implemented both SGD and Newton Descent on the mean squared loss, for the supervised learning experiments. For SGD, we used a learning rate of $0.1$ for $d=10, 100$ and a learning rate of $0.01$ for $d=1000$, and a batch size of 64. For Newton Descent, we also used a batch size of 64. To frame it as a one-step MDP, we define a reward function $r$ which is equal to the negative of mean squared loss. Both REINFORCE and ARS use this reward function. To compute the REINFORCE loss, we take the prediction of the model $\hat{w}^Tx$, add a mean $0$ standard deviation $\beta = 0.5$ Gaussian noise to it, and compute the reward (negative mean squared loss) for the noise added prediction. The REINFORCE loss function is then given by
\begin{equation}
    J(w) = \frac{1}{N} \sum_{i=1}^N r_i \frac{- (y_i - \hat{w}^Tx_i)^2}{2\beta^2}
\end{equation}
where $r_i = -(y_i - \hat y_i)^2$, $\hat y_i$ is the noise added prediction and $\hat{w}^Tx_i$ is the prediction by the model. We use an Adam optimizer with learning rate and batch size as shown in Table \ref{tab:hyperparam-reinforce}. For the natural REINFORCE experiments, we estimate the fisher information matrix and compute the descent direction by solving the linear system of equations $Fx = g$ where $F$ is the fisher information matrix and $g$ is the REINFORCE gradient. We use SGD with a $O(1/\sqrt{T})$ learning rate, where $T$ is the number of batches seen, and batch size as shown in Table \ref{tab:hyperparam-nreinforce}. 

For ARS, we closely follow the algorithm outlined in \citep{mania2018simple}.

\subsection{Multi-step Control Experiments}
\label{sec:multi-step-control-1}

\subsubsection{Tuning Hyperparameters for ARS}
\label{sec:tuning-hyperp-ars}

We tune the hyperparameters for ARS \citep{mania2018simple} in both
mujoco and LQR experiments, similar to the one-step control
experiments. The candidate hyperparameter values are shown in Tables
\ref{tab:hyperparam-multi-ars-mujoco} and \ref{tab:hyperparam-multi-ars-lqr}. We have observed that using all the
directions in ARS is always preferable under the low horizon settings
that we explore. Hence, we do not conduct a hyperparameter search over
the number of top directions and instead keep it the same as the
number of directions.

\begin{table}[ht]
    \centering
    \begin{tabular}{|c|c|c|}
    \hline
    \textbf{Hyperparameter} & \textbf{Swimmer-v2} &
                                                    \textbf{HalfCheetah-v2}\\
    \hline
    Stepsize &  $0.03, 0.05, 0.08, 0.1, 0.15$ & $0.001, 0.003,
                                                0.005,0.008, 0.01$ \\
    \hline
    \# Directions &  $5, 10, 20$ & $5, 10, 20$\\
    \hline
    Perturbation & $0.05, 0.1, 0.15, 0.2$ & $0.01, 0.03, 0.05, 0.08$\\
    \hline
    \end{tabular} 
    \caption{Candidate hyperparameters used for tuning in ARS experiments}
    \label{tab:hyperparam-multi-ars-mujoco}
\end{table}

\begin{table}[ht]
    \centering
    \begin{tabular}{|c|c|}
    \hline
    \textbf{Hyperparameter} & \textbf{LQR}\\
    \hline
    Stepsize &  $0.0001, 0.0003, 0.0005, 0.0008, 0.001, 0.003, 0.005,
               0.008, 0.01$ \\
    \hline
    \# Directions &  $10$ \\
    \hline
    Perturbation & $0.01, 0.05, 0.1$ \\
    \hline
    \end{tabular} 
    \caption{Candidate hyperparameters used for tuning in ARS experiments}
    \label{tab:hyperparam-multi-ars-lqr}
  \end{table}

We use the hyperparameters shown in Tables
\ref{tab:chosen-hyperparam-multi-ars-swimmer} and \ref{tab:chosen-hyperparam-multi-ars-halfcheetah} chosen through tuning for each
of the multi-step experiments. The hyperparameters are chosen by
averaging the total reward obtained across three random seeds
(different from the 10 random seeds used in experiments presented in
Figures \ref{fig:swimmer}, \ref{fig:halfcheetah}, \ref{fig:lqr}) and
chosing the setting that has the highest total reward after $10000$
episodes of training..

\begin{table}[ht]
    \centering
    \begin{tabular}{|c|c|c|c|}
    \hline
      \textbf{Horizon} & \textbf{Stepsize} &
                                             \textbf{\#
                                             Directions} &
                                                           \textbf{Perturbation}\\
      \hline
      $H = 1$ & 0.15  & 5  & 0.2 \\
      \hline
      $H = 2$ & 0.08  & 5 &  0.2\\
      \hline
      $H = 3$ & 0.15  & 5 &  0.2\\
      \hline
      $H = 4$ & 0.08  & 5 & 0.2 \\
      \hline
      $H = 5$ & 0.05  & 5 &  0.2\\
      \hline
      $H = 6$ & 0.08  &5  & 0.2 \\
      \hline
      $H = 7$ & 0.08  & 5 & 0.2 \\
      \hline
      $H = 8$ & 0.08  & 5 & 0.2 \\
      \hline
      $H = 9$ & 0.1  &5  & 0.2 \\
      \hline
      $H = 10$ & 0.08  &5  & 0.2 \\
      \hline
      $H = 11$ & 0.08  &5  & 0.2 \\
      \hline
      $H = 12$ & 0.1  &5  & 0.2 \\
      \hline
      $H = 13$ & 0.08  & 5 & 0.2 \\
      \hline
      $H = 14$ & 0.08  & 5 &0.2  \\
      \hline
      $H = 15$ & 0.08  & 10 & 0.2 \\
      \hline
    \end{tabular} 
    \caption{Hyperparameters chosen for multi-step experiments for ARS
    in Swimmer-v2}
    \label{tab:chosen-hyperparam-multi-ars-swimmer}
\end{table}

\begin{table}[ht]
    \centering
    \begin{tabular}{|c|c|c|c|}
    \hline
      \textbf{Horizon} & \textbf{Stepsize} &
                                             \textbf{\#
                                             Directions} &
                                                           \textbf{Perturbation}\\
      \hline
      $H = 1$ & 0.001  & 20 & 0.08 \\
      \hline
      $H = 2$ & 0.008  & 5 &  0.08\\
      \hline
      $H = 3$ &  0.008  & 10 & 0.08 \\
      \hline
      $H = 4$ &  0.003  & 5 &  0.05\\
      \hline
      $H = 5$ &  0.003  & 5 &  0.05\\
      \hline
      $H = 6$ &  0.003  & 10 &  0.05\\
      \hline
      $H = 7$ &  0.008  & 20 &  0.05\\
      \hline
      $H = 8$ &  0.008  & 5 &  0.05\\
      \hline
      $H = 9$ &  0.01  & 20 &  0.03\\
      \hline
      $H = 10$ &   0.005 & 10 &  0.03\\
      \hline
      $H = 11$ &  0.008  & 20 &  0.03\\
      \hline
      $H = 12$ &  0.005  & 5 &  0.05\\
      \hline
      $H = 13$ &  0.008  & 20 &  0.03\\
      \hline
      $H = 14$ &  0.01  & 10 &  0.03\\
      \hline
      $H = 15$ &  0.008  & 20 &  0.03\\
      \hline
    \end{tabular} 
    \caption{Hyperparameters chosen for multi-step experiments for ARS
    in HalfCheetah-v2}
    \label{tab:chosen-hyperparam-multi-ars-halfcheetah}
\end{table}

\subsubsection{Tuning Hyperparameters for ExAct}
\label{sec:tuning-hyperp-exact}

We tune the hyperparameters for ExAct (Algorithm
\ref{alg:random_search_action}) in both mujoco and LQR experiments,
similar to ARS. The candidate hyperparameter values are shown in
Tables \ref{tab:hyperparam-multi-exact-mujoco} and
\ref{tab:hyperparam-multi-exact-lqr}. Similar to ARS, we do not
conduct a hyperparameter search over the number of top directions and
instead keep it the same as the number of directions.

\begin{table}[H]
    \centering
    \begin{tabular}{|c|c|c|}
    \hline
    \textbf{Hyperparameter} & \textbf{Swimmer-v2} &
                                                    \textbf{HalfCheetah-v2}\\
    \hline
    Stepsize &  \thead{$0.005, 0.008, 0.01, 0.015,$\\$0.02, 0.025, 0.03$} & \thead{$0.0001, 0.0003,
                                                0.0005,0.0008,$\\$0.001,
                                                                 0.002,
                                                                 0.003$} \\
    \hline
    \# Directions &  $5, 10, 20$ & $5, 10, 20$\\
    \hline
    Perturbation & $0.15, 0.2, 0.3, 0.5$ &  $0.15, 0.2, 0.3, 0.5$\\
    \hline
    \end{tabular} 
    \caption{Candidate hyperparameters used for tuning in ExAct experiments}
    \label{tab:hyperparam-multi-exact-mujoco}
\end{table}

\begin{table}[ht]
  \centering
  \begin{tabular}{|c|c|}
    \hline
    \textbf{Hyperparameter} & \textbf{LQR}\\
    \hline
    Stepsize &  $0.0001, 0.0003, 0.0005, 0.0008, 0.001, 0.003, 0.005,
               0.008, 0.01$ \\
    \hline
    \# Directions &  $10$ \\
    \hline
    Perturbation & $0.01, 0.05, 0.1$ \\
    \hline
  \end{tabular} 
  \caption{Candidate hyperparameters used for tuning in ExAct experiments}
  \label{tab:hyperparam-multi-exact-lqr}
\end{table}

We use the hyperparameters shown in Tables
\ref{tab:chosen-hyperparam-multi-exact-swimmer} and \ref{tab:chosen-hyperparam-multi-exact-halfcheetah} chosen through tuning for
each of the multi-step experiments, similar to ARS.

\begin{table}[ht]
    \centering
    \begin{tabular}{|c|c|c|c|}
    \hline
      \textbf{Horizon} & \textbf{Stepsize} &
                                             \textbf{\#
                                             Directions} &
                                                           \textbf{Perturbation}\\
      \hline
      $H = 1$ & 0.02  &5  & 0.2 \\
      \hline
      $H = 2$ & 0.02  & 5 & 0.2 \\
      \hline
      $H = 3$ & 0.015  & 10 & 0.2 \\
      \hline
      $H = 4$ & 0.015  & 10 & 0.2 \\
      \hline
      $H = 5$ & 0.01  & 10 & 0.2 \\
      \hline
      $H = 6$ & 0.015  & 10 & 0.2 \\
      \hline
      $H = 7$ & 0.01  & 20 & 0.2 \\
      \hline
      $H = 8$ & 0.015  & 20 & 0.2 \\
      \hline
      $H = 9$ & 0.02  & 20 & 0.2 \\
      \hline
      $H = 10$ & 0.008  & 5 & 0.2 \\
      \hline
      $H = 11$ & 0.02  & 5 & 0.15 \\
      \hline
      $H = 12$ & 0.02  & 20 & 0.2 \\
      \hline
      $H = 13$ & 0.015  & 5 & 0.15 \\
      \hline
      $H = 14$ & 0.02  & 10 &0.15  \\
      \hline
      $H = 15$ & 0.01  & 5 & 0.1 \\
      \hline
    \end{tabular} 
    \caption{Hyperparameters chosen for multi-step experiments for ExAct
    in Swimmer-v2}
    \label{tab:chosen-hyperparam-multi-exact-swimmer}
\end{table}

\begin{table}[ht]
    \centering
    \begin{tabular}{|c|c|c|c|}
    \hline
      \textbf{Horizon} & \textbf{Stepsize} &
                                             \textbf{\#
                                             Directions} &
                                                           \textbf{Perturbation}\\

      \hline
      $H = 1$ &0.0001   &20  &  0.2\\
      \hline
      $H = 2$ &  0.001 & 5 &  0.2\\
      \hline
      $H = 3$ &  0.001  & 5 & 0.2\\
      \hline
      $H = 4$ &  0.001  & 5 & 0.2 \\
      \hline
      $H = 5$ &  0.001  &10  & 0.2 \\
      \hline
      $H = 6$ &  0.001  & 5 & 0.2 \\
      \hline
      $H = 7$ &  0.001  &10  & 0.2 \\
      \hline
      $H = 8$ &  0.001  & 5 & 0.2 \\
      \hline
      $H = 9$ &  0.001  & 5 & 0.2 \\
      \hline
      $H = 10$ &  0.001  & 5 & 0.2 \\
      \hline
      $H = 11$ & 0.0008   & 5 & 0.15 \\
      \hline
      $H = 12$ &  0.001  & 5 & 0.2\\
      \hline
      $H = 13$ &  0.001  & 10 & 0.2 \\
      \hline
      $H = 14$ &  0.001  & 5 & 0.2\\
      \hline
      $H = 15$ &  0.0008  & 10 & 0.2 \\
      \hline
    \end{tabular} 
    \caption{Hyperparameters chosen for multi-step experiments for ExAct
    in HalfCheetah-v2}
    \label{tab:chosen-hyperparam-multi-exact-halfcheetah}
\end{table}

\subsubsection{Mujoco Experiments}
\label{sec:mujoco-experiments}

For all the mujoco experiments, both ARS and ExAct use a linear
policy with the same number of parameters as the dimensionality of the
state space. The hyperparameters for both algorithms are chosen as
described above. Each algorithm is run on both
environments (Swimmer-v2
and HalfCheetah-v2) for $10000$ episodes of training across $10$
random seeds (different from the ones used for tuning). This is
repeated for each horizon value $H \in \{1, 2, \cdots, 15\}$. In each
experiment, we record the mean evaluation return obtained after
training and plot the results in Figures \ref{fig:swimmer},
\ref{fig:halfcheetah}. For more details on the environments used, we
refer the reader to \citep{brockman2016openai}.

\subsubsection{LQR Experiments}
\label{sec:lqr-experiments}

In the LQR experiments, we constructed a linear dynamical system
$x_{t+1} = Ax_t + Bu_t + \xi_t$ where $x_t \in \mathbb{R}^{100}$, $A \in \mathbb{R}^{100\times
  100}$, $B \in \mathbb{R}^{100}$, $u_t \in
\mathbb{R}$ and the noise $\xi_t \sim \mathcal{N}(0_{100}, cI_{100
  \times 100})$ with a small constant $c \in \mathbb{R}^+$. We
explicitly make sure that the maximum eigenvalue of $A$ is less than 1
to avoid instability. We fix a quadratic cost function $c(x, u) =
x^TQx + uRu$, where $Q = 10^{-3}I_{100 \times 100}$ and $R = 1$. The
hyperparameters chosen for both algorithms are chosen as described
above.

For each algorithm, we run it for noise covariance values $c \in
\{10^{-4}, 5\times 10^{-4}, 10^{-3}, 5\times 10^{-3}, 10^{-2}, 5\times 10^{-2},
10^{-1}, 5\times 10^{-1}\}$
until we reach a stationary point where $\|\nabla_\theta
J(\theta)\|_2^2 \leq 0.05$. The number of interactions with the
environment allowed is capped at $10^6$ steps for each run. This is
repeated across $10$ random seeds (different from the ones used for
tuning). The number of interactions needed to reach the stationary
point as the noise covariance is increased is recorded and shown in
Figure \ref{fig:lqr}.

\end{document}